\documentclass{article} 
\usepackage{iclr2027_conference,times}

\usepackage{amsmath,amsfonts,bm}

\def\eqref#1{equation~\ref{#1}}

\def\1{\bm{1}}

\DeclareMathAlphabet{\mathsfit}{\encodingdefault}{\sfdefault}{m}{sl}
\SetMathAlphabet{\mathsfit}{bold}{\encodingdefault}{\sfdefault}{bx}{n}

\usepackage[pagebackref=true,colorlinks,linkcolor=blue,citecolor=brown]{hyperref}

\iclrfinalcopy

\usepackage[utf8]{inputenc} 
\usepackage[T1]{fontenc}    
\usepackage{url}            
\usepackage{booktabs}       
\usepackage{amsfonts}       
\usepackage{nicefrac}       
\usepackage{microtype}      
\usepackage{xcolor}         
\usepackage{amsmath}
\usepackage{amssymb}
\usepackage{amsthm}
\usepackage{algorithm}
\usepackage{algpseudocode}
\usepackage{graphicx}
\usepackage{tabularx}
\usepackage{listings}
\usepackage{etoc}
\etocsettagdepth{mainpaper}{none}
\etocsettagdepth{appendix}{subsection}
\usepackage{thm-restate}
\usepackage{wrapfig}
\usepackage{forest} 
\useforestlibrary{edges}
\usepackage{pifont}
\usepackage{multirow}
\usepackage[table]{xcolor}
\definecolor{cellbase}{HTML}{F9F9F9}
\definecolor{tabband}{HTML}{F2F7FB}
\newtheorem{theorem}{Theorem}[section]
\newtheorem{lemma}[theorem]{Observation}

\newtheorem{corollary}[theorem]{Observation}

\newtheorem{assumption}[theorem]{Assumption}
\newtheorem{motivation}[theorem]{Motivation}
\newcommand{\ours}{\texttt{BES}}
\usepackage{caption} 
\usepackage{xcolor}
\definecolor{darkerlogocolor}{RGB}{20, 0, 145}
\usepackage{listings}
\usepackage[most]{tcolorbox}
\newtcolorbox{ttcolorbox}[1][]{
  colframe=darkerlogocolor,
  colback=darkerlogocolor!4!white,
  title=#1,
  breakable,
  enhanced jigsaw,
}

\lstdefinestyle{prompt}{
  basicstyle=\small\ttfamily,
  breaklines=true,
  breakatwhitespace=false,
  columns=fullflexible,
  keepspaces=true,
  showstringspaces=false,
  upquote=true,
  xleftmargin=0pt,
  xrightmargin=0pt,
}
\usepackage[svgnames]{xcolor} 
\usepackage{amsthm}           
\usepackage[most]{tcolorbox}  
\tcolorboxenvironment{motivation}{
    colback=LightBlue!20,
    colframe=black,
    arc=1mm,
    boxrule=.5pt,
    top=.6mm, bottom=.6mm, left=1mm, right=1mm,
    before skip=6pt, after skip=6pt,
}

\theoremstyle{remark}

\title{Self-Improving Language Models with \\ Bidirectional Evolutionary Search}

\author{%
  Guowei Xu\textsuperscript{1} \quad
  Zhenting Qi\textsuperscript{1} \quad
  Huangyuan Su\textsuperscript{1} \quad
  Weirui Ye\textsuperscript{2}  \\[2pt]
  \textbf{Himabindu Lakkaraju\textsuperscript{1}} \quad
  \textbf{Sham M.\ Kakade\textsuperscript{1}} \quad
  \textbf{Yilun Du\textsuperscript{1}} \\[6pt]
  \textnormal{\textsuperscript{1}Harvard University \qquad
  \textsuperscript{2}MIT}
}

\begin{document}

\maketitle
\etocdepthtag.toc{mainpaper}

\vspace{-2em}
\begin{abstract}
  Search has been proposed as an effective method for self-improving language models and agentic systems, both for post-training sample generation and for inference. However, widely used methods such as best-of-N sampling and tree search face two fundamental limitations: they are guided by sparse verification signals, and they construct candidates primarily through autoregressive expansion, restricting exploration to regions with substantial model probability mass. To address these, we propose Bidirectional Evolutionary Search (\ours), a search framework that couples forward candidate evolution with backward goal decomposition. In the forward search, \ours\ augments standard expansion with evolution operators that recombine partial trajectories to generate candidates that are difficult to obtain from a single model rollout. In the backward search, \ours\ recursively decomposes the original task into checkable sub-goals, producing dense intermediate feedback that guides forward search. We provide motivation illustrating that candidates generated by expansion-only search may be confined to a narrow entropy shell, while evolution operators may escape it.
  Experiments show that on challenging post-training tasks where mainstream post-training algorithms fail to improve, \ours\ enables consistent gains. At inference time, \ours\ outperforms existing open-source frameworks on three open problem solving benchmarks in both average and best-case performance, and outperforms majority voting and Monte Carlo Tree Search on mathematical reasoning and competitive programming benchmarks. 
  Code and trained models are available at \url{https://github.com/Embodied-Minds-Lab/BES}.
\end{abstract}
\vspace{-1em}
\section{Introduction}
Large language models (LLMs) and agentic systems have demonstrated remarkable capabilities on complex reasoning problems~\citep{rstar, dsr1}. They can even solve open problems across mathematical and scientific domains~\citep{alphaevolve,alphaproof} and surpass the best human performance on tasks such as code generation~\citep{alphacode, tttdiscover}. In this context, the question of how to do better sampling from LLMs and agentic systems is of critical importance~\citep{rstarmath}. This is particularly significant for problems at the frontier of model capability, where naive sampling methods may require too many samples to obtain a correct answer or may simply fail~\citep{frontiermath,rlincentivize}. At training time, higher-quality samples enable more effective post-training and self-improvement~\citep{star,selfreward}; at inference time, they serve as a natural mechanism for test-time scaling~\citep{selftaught, shinkaevolve}, which can further push the boundary of what models can achieve.

Currently, the two dominant sampling methods in post-training, self-improvement, and inference for LLMs and agentic systems are best-of-N sampling and tree search. Best-of-N sampling is simple and efficient. For problems of moderate difficulty, it typically suffices to find high-quality responses and is therefore widely adopted in post-training algorithms such as GRPO~\citep{dsr1} and its variants~\citep{dapo}, while also serving as a strong baseline for inference. Tree search methods such as beam search and Monte Carlo Tree Search~\citep{mcts} can discover better responses more sample-efficiently than best-of-N on harder problems. For example, Tree-GRPO leverages tree search for sample generation during post-training~\citep{treegrpo}, and Tree of Thoughts explores multiple reasoning paths at inference time~\citep{tot}.

However, both methods share two limitations.  (1) The verification signal to guide the search is sparse. Effective search depends on the accuracy and granularity of the verifier, yet in common settings such as RLVR post-training, verifiers typically provide only binary or coarse-grained feedback. (2) They struggle to generate candidates beyond the model's own distribution. They construct candidates by auto-regressively extending responses. This concentrates candidates in the high-probability region of the model's distribution, so trajectories the model assigns low probability are rarely sampled~\citep{llmmonkey}, even though they may be needed on hard problems.

\begin{figure}[t]
\centering
\vspace{-1em}
\includegraphics[width=\linewidth]{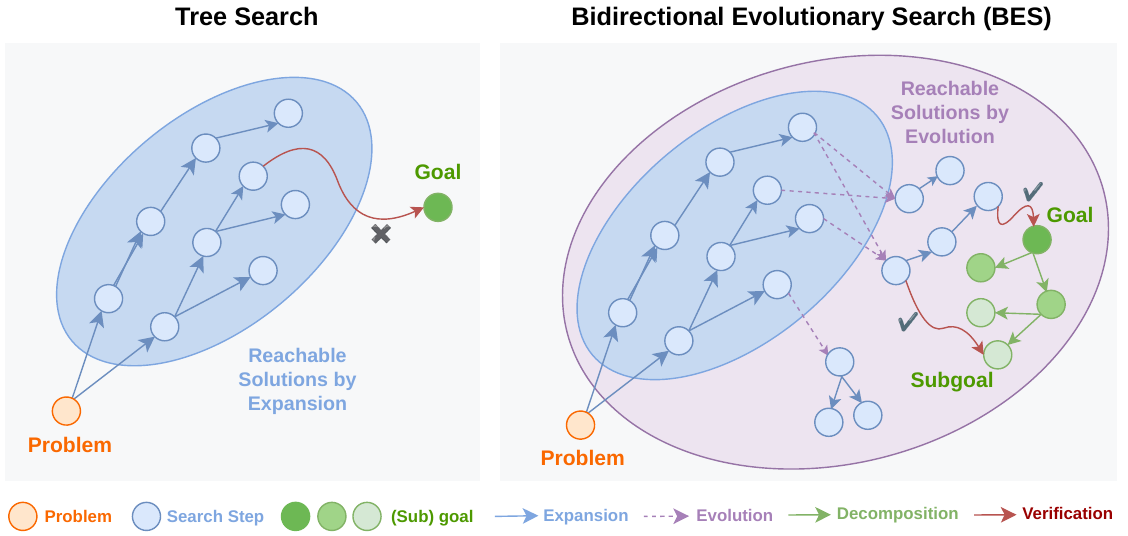}
\vspace{-1em}
\caption{Comparison of tree search and Bidirectional Evolutionary Search (\ours). \textbf{Left:} Tree search constructs candidates by sequentially expanding steps. Our motivating analysis suggests that such candidates are typically confined to a narrow entropy shell, limiting exploration to a small region of the solution space. \textbf{Right:} Evolution operators can produce candidates outside this shell, with backward search decomposing the
problem into verifiable sub-goals that provide dense feedback to guide
the forward search. \textcolor{green!60!black}{\checkmark} and {\color{red!70!black}$\boldsymbol{\times}$} indicate whether a candidate satisfies or fails the (sub-)goal.}
\vspace{-2em}
\label{fig:teaser}
\end{figure}

To address these two limitations, we propose  Bidirectional Evolutionary Search (\ours). First, to tackle the sparsity of verification signals, we introduce bidirectional search: the forward search seeks better candidate solutions, while the backward search discovers finer-grained sub-goals to verify them.
Second, to generate candidates beyond the model's own distribution, we draw inspiration from evolutionary biology. 
For much of the history of life, organisms reproduced asexually. Each offspring was a direct extension of its parent, and beneficial mutations arising independently in different individuals could never be combined~\citep{naturalselection}. Sexual reproduction fundamentally changed this through chromosomal recombination: gene segments from different lineages are spliced together to produce novel combinations that neither parent possessed~\citep{sex}. Analogously, in our setting, rather than only extending responses auto-regressively, we introduce four evolution operators: combination, translocation, deletion, and crossover. Combination, translocation, and crossover merge the strengths of two distinct responses in different ways to construct a new candidate, while deletion removes the least sound segment from a response. We give a motivating analysis suggesting that responses generated by expansion-only search may be confined to a narrow entropy shell, while evolution operators may escape it.

We evaluate \ours\ on both post-training and inference across LLM and agent settings. For post-training, we test on challenging logical reasoning and multi-hop reasoning tasks where mainstream post-training algorithms such as GRPO~\citep{dsr1}, MaxRL~\citep{maxrl}, and Tree-GRPO~\citep{treegrpo} struggle to find sufficient high-quality training samples and consequently fail to improve or even degrade from the base model. In contrast, \ours\ consistently discovers effective training samples, enabling meaningful improvements where these baselines nearly fail. For inference, we test on three open problem solving benchmarks, where \ours\ discovers more stable and higher-quality solutions compared to all existing open-source frameworks, including OpenEvolve~\citep{openevolve}, GEPA~\citep{gepa}, and ShinkaEvolve~\citep{shinkaevolve}. We further evaluate \ours\ on mathematical reasoning (AIME~\citep{aime25}) and competitive programming (LiveCodeBench~\citep{livecodebench}) benchmarks, where \ours\ outperforms majority voting and Monte Carlo Tree Search~\citep{mcts} under the same compute budget. We also provide ablation studies validating the contribution of each component, case studies visualizing our search process, and cost analysis comparing the wall-clock time of \ours\ against baselines.

\section{Preliminaries}
\label{sec:problem}

We consider reasoning problems of the form
$
  \mathcal{T} = (x, V),
$
where $x$ is a problem description and
$V(x,y) \in [0,1]$ is a \emph{verifier} that
assigns a score measuring how well a trajectory $y$
solves problem $x$.

Given a policy $\pi_\theta(\cdot\mid x)$ (e.g.\ an LLM that generates reasoning traces, or an agent that interacts with an environment through a sequence of actions), our goal is to produce a terminal response $y$ that maximizes the verifier score. Let
$\mathcal{Y}_{\mathrm{term}}(x)$ denote the set of valid terminal responses
under the task specification and resource budget. We write the target as
\(
  y^{\star}(x) \;\in\; \operatorname*{arg\,max}_{y\in\mathcal{Y}_{\mathrm{term}}(x)}
                       V(x, y).
\)
Finding $y^{\star}$ is important in both training and inference. During
training, high-quality candidates can facilitate post-training or
iterative self-improvement. During inference, $y^{\star}$ is the response
we seek.
The challenge is that the maximization is intractable:
on hard problems, the probability mass that $\pi_\theta$ assigns to correct
trajectories can be extremely small. Practical algorithms therefore
approximate $y^{\star}$ by searching over a set of candidates. We 
introduce two common methods below.

\textbf{Best-of-$N$ sampling.}
The simplest approach is to draw $N$ independent trajectories and return the one with the highest verifier score. Formally, let $y^{(1)}, \dots, y^{(N)} \overset{\text{i.i.d.}}{\sim} \pi_\theta(\cdot \mid x)$. The method returns $ y^{\mathrm{BoN}} \;=\; \operatorname*{arg\,max}_{y \in \{y^{(1)}, \dots, y^{(N)}\}} V(x, y) $.
Best-of-$N$ is parallel and requires no structural knowledge of the problem. Its effectiveness, however, is bounded by the finite-sample coverage of $\pi_\theta$: all $N$ trajectories are drawn from the same distribution, so if the optimal trajectory lies in a region with very small policy mass, increasing $N$ gives only linear coverage improvement.

\textbf{Tree search.}
Tree-search methods exploit the sequential structure of a trajectory. A trajectory
is decomposed into steps, and the
search maintains a tree whose nodes are partial trajectories and whose edges
correspond to appending a step. Branches are selected and extended according
to a heuristic value so that compute is concentrated on prefixes that look
promising. Classic methods include beam search, best-first search, and Monte
Carlo Tree Search~\citep{mcts}; recent work applies these ideas to LLM and agent reasoning~\citep{tot,treegrpo}. Tree search can be more sample-efficient than best-of-$N$
when the per-step signal is informative, but it still materializes terminal
candidates one lineage at a time through sequential extension.

\section{\ours: Bidirectional Evolutionary Search}
\label{sec:method}

 \ours\ performs a bidirectional evolutionary search that alternates
between two coupled processes: a \emph{forward} search that seeks
better candidates, and a \emph{backward} search that decomposes
the problem into fine-grained sub-goals to evaluate each forward
node. The full pseudocode is given in
Appendix~\ref{app:algo}, and a case study illustrating the search
process is provided in Appendix~\ref{app:case}.

\subsection{Forward Search: Expanding the Reachable Solution Space}
\label{sec:forward}

\begin{figure}[t]
\centering
\includegraphics[width=\linewidth]{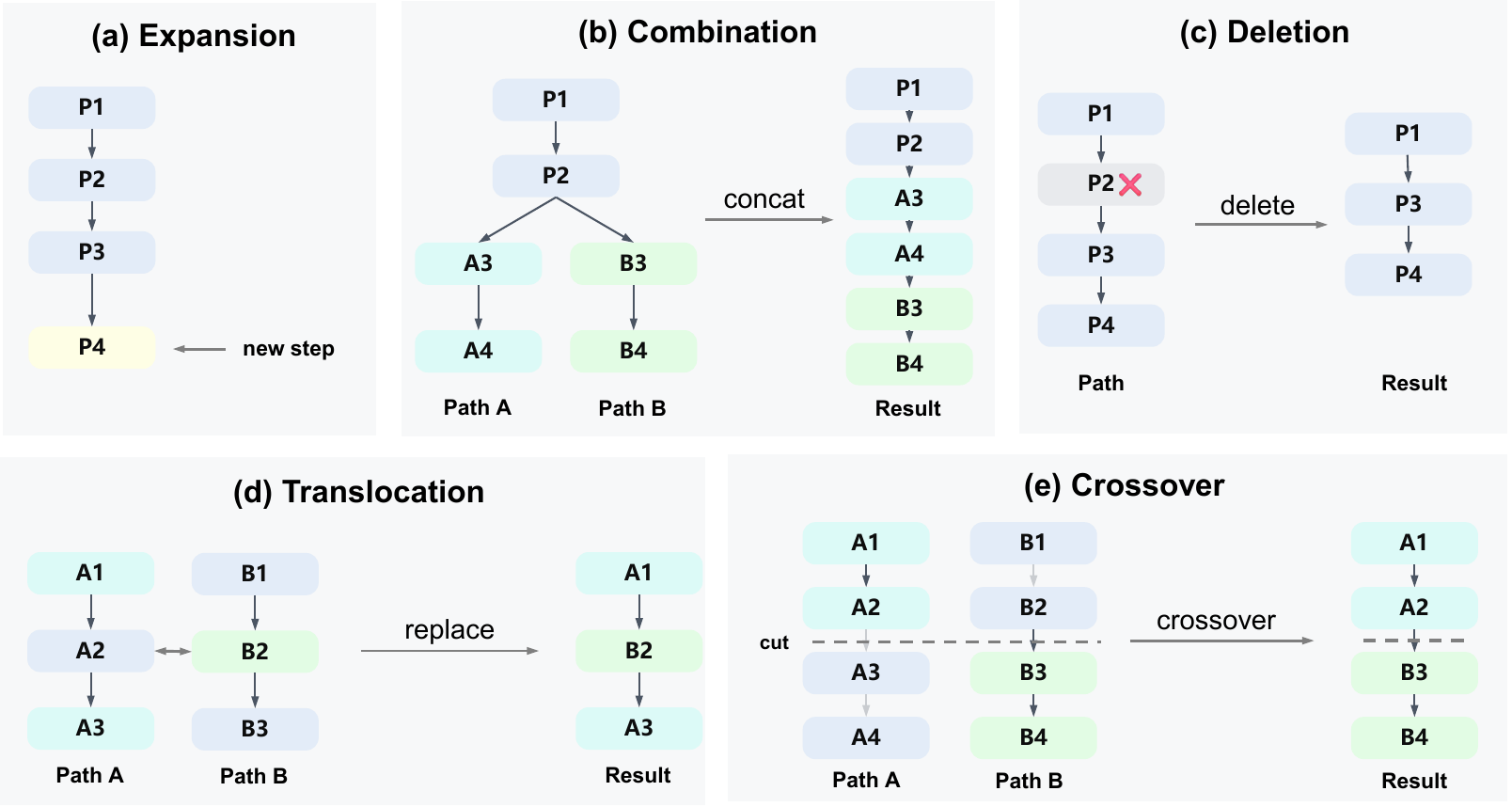}
\caption{Forward search operators. \textbf{(a) Expansion:} the policy generates new steps (yellow). \textbf{(b) Combination:} beyond their shared part (P1--P2), the contents of Path A and Path B are merged into a single candidate. \textbf{(c) Deletion:} one step (P2) is removed. \textbf{(d) Translocation:} a useful step from Path B (B2) is transplanted into Path A, replacing A2, while the rest of Path A is kept. \textbf{(e) Crossover:} the beginning of Path A (A1--A2) is continued with the latter part of Path B (B3--B4). The operators are realized by prompting when direct edits are not meaningful.}
\vspace{-1em}
\label{fig:forward}
\end{figure}

We represent each candidate partial trajectory as a node
$n=(y_1,\dots,y_{t})$, where $y_i$ denotes the $i$-th step (e.g.\ a reasoning segment or an action). The search maintains a candidate set $\mathcal{P}$ of such nodes, and at
each search step applies one of two types of operators, \emph{expansion}
or \emph{evolution}, to produce a child node $n'$, which is scored by
the backward search (Section~\ref{sec:backward}) and added to $\mathcal{P}$.

\textbf{Expansion.}
Expansion extends a parent node by sampling new steps. Given
$n=(y_1,\dots,y_{t})$, we sample a step count
$K\sim\mathrm{Uniform}\{1,\dots,K_{\max}\}$ and draw up to $K$ new steps from
$\pi_\theta$:
\begin{equation}
  y_{t+k}\,\sim\,\pi_\theta\bigl(\cdot\mid x\oplus y_1\oplus\cdots\oplus y_{t+k-1}\bigr),
  \qquad k=1,\dots,K,
  \label{eq:expand}
\end{equation}

\textbf{Evolution.}
A key limitation of expansion alone is that each candidate is built by
sequentially extending a single trajectory, so it cannot combine useful
parts from different candidates. Evolution operators overcome this by
constructing a new candidate from the content of existing ones. As
illustrated in Figure~\ref{fig:forward}, we define four operators
inspired by biological evolution, each specified by what it takes from
which parent: \textit{(i) Combination} produces a candidate that draws on
the content of two trajectories beyond what they already share;
\textit{(ii) Deletion} removes one segment of a trajectory,
producing a condensed candidate; \textit{(iii) Translocation} transplants a
single useful segment from a donor trajectory into a recipient ; and \textit{(iv) Crossover}
continues the beginning of one trajectory with the latter part of another.
Together, these operators allow the search to restructure, condense, and
recombine existing trajectories. As with mutations in nature, evolution
operations are not guaranteed to produce better candidates. However, the
value of evolution lies in generating diverse candidates:
\textbf{\textit{even if only a small fraction turn out to be improvements,
that is sufficient for the search to make progress.}}
In practice, we instantiate the operators
either as direct edits when steps are self-contained
with clear boundaries (Appendix~\ref{app:operators}), or by prompting the
model when direct edits are not meaningful.

At each forward search step, we select among the operators with fixed
probabilities. All operators require selecting one or two parent nodes from
$\mathcal{P}$. Let $\mathcal{C}_u$ be the set of eligible members of
$\mathcal{P}$ at step $u$. For single-parent operators (expansion, deletion), we sample the parent from a Boltzmann distribution over backward score $s(n)$ (Eq.~\ref{eq:score}):
\begin{equation}
  \Pr\bigl[n\mid\mathcal{C}_u\bigr]
    \;=\;\frac{\exp(\tilde{s}(n)/\tau_u)}
              {\sum_{n'\in\mathcal{C}_u}\exp(\tilde{s}(n')/\tau_u)},
  \qquad
  \tilde{s}(n) \;=\; s(n) + \lambda\cdot\mathbf{1}\!\bigl[\deg(n)=0\bigr],
  \label{eq:boltz}
\end{equation}
where $\tau_u>0$, $\deg(n)$ is the number of children of $n$ in
$\mathcal{P}$, and the indicator term adds a small constant bonus
$\lambda=0.1$ to candidates that have not yet been
selected as parents.

For two-parent operators, we select a pair of parent nodes $(n_a, n_b)$ from $\mathcal{C}_u$.
We calculate a pair score $s(n_a, n_b)$ (Eq.~\ref{eq:pair_score}) that favors complementary parents whose strengths cover different parts of the problem.
The pair is then drawn from the analogous Boltzmann distribution:
\begin{equation}
  \Pr[(n_a, n_b) \mid \mathcal{C}_u]
   \;=\;\frac{\exp(s(n_a, n_b)/\tau_u)}
              {\sum_{n_a, n_b\in\mathcal{C}_u}\exp(s(n_a, n_b)/\tau_u)}.
        \label{eq:bi_boltz}
\end{equation}
The temperature $\tau_u$ is annealed linearly over the search, shifting from exploration to exploitation.

We give an intuition for why evolution operators may produce candidates that expansion rarely generates.
For simplicity, we study the direct-edit instantiation of the evolution operators.
Fix a task $x$ and a step horizon $T$. A trajectory $Y = (y_1, \dots, y_T)$ is
generated by the policy $\pi_\theta$, which induces the trajectory distribution
$P(Y) := \pi_\theta(Y \mid x)$. Let $H_T = H_P(Y)$ denote the trajectory-level entropy, and define,
for a constant $\epsilon > 0$, the
typical shell $A^{(T)} := \{y : |-\log P(y) - H_T| \leq \epsilon T\}$. Under
assumptions on token dependence and next-token uncertainty
(Appendix~\ref{app:assumptions}), we have:

\begin{motivation}[Shell confinement and escape]
\label{thm:shell}
\textbf{(a)} A policy rollout $Y \sim P$ stays within the shell with high probability:
$\Pr[Y \notin A^{(T)}] \leq \exp(-\Omega(T))$.
\textbf{(b)} A candidate $\widetilde{Y}$ formed by splicing segments of independent
rollouts can escape it: $\Pr[\widetilde{Y} \notin A^{(T)}] > 0$.
\end{motivation}

Escaping the shell indicates that such candidates are generated with non-negligible probability, not that they are correct. Identifying
useful candidates among them is the role of the backward search. Details
are illustrated in Appendix~\ref{app:theory-mutation}.

\subsection{Backward Search: Better Verification through Goal Decomposition}
\label{sec:backward}

While the verifier $V$ provides a score for each node in the forward search, this signal is relatively sparse. Backward search addresses this by decomposing
the problem into a tree of fine-grained sub-goals. Each forward
node is then scored against this tree: the more sub-goals a
candidate trajectory has addressed, the higher its score.
This gives the forward search a dense, informative signal to
select promising candidates, even when none of them have fully
solved the problem yet. 

Below we describe how the backward search is constructed.
Starting from the top-level goal $g_{\texttt{root}}$ (i.e.\ solving the
entire problem), the policy $\pi_\theta$ is prompted
to break each goal into finer sub-goals, producing a rooted backward goal tree. Each goal $g$ on the tree can be recursively split into children
$\mathrm{ch}(g)$ (finer sub-goals), and every $g$ comes with a verifier
$V_g(x,n) \in [0,1]$ that tests
how well a candidate node $n$ addresses the sub-goal $g$ on
problem $x$. 
For the top-level goal, $V_{g_{\mathrm{root}}} = V$ (the verifier of the original
problem when available, and a model-generated verifier otherwise).
This decomposition is re-invoked every $K_{dec}$ forward
search steps: at each invocation, we select a leaf sub-goal that no current candidate fully satisfies and prompt $\pi_\theta$ to split it into finer sub-goals. After that, all existing forward nodes are re-scored. 
The verifiers are task-dependent and can be instantiated
as rule-based checkers, test-case code executors, embedding similarity
models, or LLM judgers. We describe the verifiers used in each experiment
in Appendix~\ref{app:setup}.

For example, consider the problem ``Compute
$\frac{(4+6)\times 3}{2} - 5$.'' The backward search
can produce:
\vspace{-0.8em}
\begin{center}
\begin{forest}
  for tree={
    grow'=0,
    anchor=west,
    child anchor=west,
    parent anchor=east,
    l sep=4mm,
    s sep=0.5mm,
    inner xsep=1pt,
    inner ysep=1pt,
    edge={->,>=latex},
    font=\footnotesize,
  },
  forked edges,
  [$g_{\texttt{root}}$: compute $\frac{(4+6)\times 3}{2} - 5$
    [$g_1$: compute $(4+6)\times 3$
      [$g_{1.1}$: compute $4+6$]
      [$g_{1.2}$: \#1 multiply by $3$]
    ]
    [$g_2$: \#1 divide by $2$]
    [$g_3$: \#2 subtract $5$]
  ]
\end{forest}
\end{center}
\vspace{-0.4em}
For a candidate node $n$ from the forward search and a sub-goal $g$, define the
sub-goal score
\begin{equation}
  s(n,g) \;=\; \alpha\cdot V_g(x,n)
    \;+\;(1-\alpha)\cdot\frac{1}{|\mathrm{ch}(g)|}
    \sum_{g'\in\mathrm{ch}(g)} s(n,g'),
  \label{eq:score}
\end{equation}
where $\alpha\in[0,1]$ balances the contribution of parent goals and
 sub-goals. For leaf sub-goals ($\mathrm{ch}(g)=\emptyset$),
we set $s(n,g)=V_g(x,n)$. If a goal is already fully
satisfied ($V_g(x,n)=1$), we short-circuit to $s(n,g)=1$ without
evaluating its children. 
The overall score is $s(n)\triangleq s(n,g_{\texttt{root}})$.

For two candidate nodes $n_a$ and $n_b$, we further define a pair score that measures their joint coverage of the goal tree.
Replacing each sub-goal verifier with the maximum of the two
parents' scores gives
\begin{equation}
  s(n_a,n_b,g) \;=\; \alpha\cdot \max\{V_g(x,n_a),\, V_g(x,n_b)\}
    \;+\;(1-\alpha)\cdot\frac{1}{|\mathrm{ch}(g)|}
    \sum_{g'\in\mathrm{ch}(g)} s(n_a,n_b,g'),
  \label{eq:pair_score}
\end{equation}
with $s(n_a,n_b)\triangleq s(n_a,n_b,g_{\texttt{root}})$. A sub-goal is
considered covered if either parent addresses it, so the pair
score favors complementary parents that cover different parts
of the goal tree.

Applying this procedure to every node in the forward search
tree yields a backward score for each node and a pair
score for each pair of nodes, which directly drive node
selection in the forward search.

\subsection{Using \ours\ for Post-Training and Inference}
\label{sec:usage}

\textbf{Post-Training.}
\ours\ replaces the sample-generation of post-training. Given a problem, \ours\ returns a set of high-quality trajectories. These are used as training data for post-training, replacing candidates that would otherwise be obtained from i.i.d.\ rollouts.
Because \ours\ can find correct solutions that single rollouts
rarely reach, it provides stronger training signal, especially
on hard problems. Here, the verifier is available from training labels and is used to identify successful trajectories.

\textbf{Inference.} At inference time, \ours\ runs with a fixed compute budget and returns the highest-scoring terminal trajectory as the final answer. 
For mathematical reasoning and code generation, where a ground-truth verifier
is unavailable to the model, the policy writes a verifier for the top-level
goal, and \ours\ returns the terminal trajectory that scores
highest under it.

\clearpage
\section{Experiments}
We evaluate \ours\ on both post-training and inference across LLM and agent settings. For post-training, we consider multi-hop reasoning and logical reasoning. For inference, we consider open problem solving, mathematical reasoning, and competitive programming. In each benchmark, we compare \ours\ against commonly used baselines and show that \ours\ consistently achieves better sampling.

\subsection{Bidirectional Evolutionary Search for Post-Training}

\subsubsection{Multi-Hop Reasoning}
\label{sec:multihop}

\begin{table*}[h]
\centering
\renewcommand{\arraystretch}{1.15}
\caption{Multi-hop reasoning post-training results. We report accuracy (\%) on the MuSiQue validation set (in-distribution) and on Bamboogle (out-of-distribution); $\pm$ denotes the binomial standard error. \textbf{Bold} denotes the best method for each backbone and test set.}
\vspace{4pt}
\label{tab:multihop}

\begin{tabularx}{\linewidth}{@{} l >{\centering\arraybackslash}X cc @{}}
\toprule
\textbf{Method} & \textbf{Test Dataset} & \textbf{Llama-3.2-3B-Instruct} & \textbf{Llama-3.1-8B-Instruct} \\
\midrule
\rowcolor{cellbase} & MuSiQue (ID)  & 4.01 $\pm$ 0.40 & 6.58 $\pm$ 0.50  \\
\rowcolor{cellbase} \multirow{-2}{*}{Base model} & Bamboogle (OOD) & 24.0 $\pm$ 3.8 & 22.4 $\pm$ 3.7  \\
\midrule[\lightrulewidth]
\rowcolor{cellbase} & MuSiQue  (ID) & 2.11 $\pm$ 0.29 & 5.59 $\pm$ 0.47  \\
\rowcolor{cellbase} \multirow{-2}{*}{GRPO} & Bamboogle (OOD)& 16.8 $\pm$ 3.3 & 20.0 $\pm$ 3.6  \\
\midrule[\lightrulewidth]
\rowcolor{cellbase} & MuSiQue (ID)  & 3.93 $\pm$ 0.40 & 7.45 $\pm$ 0.53  \\
\rowcolor{cellbase} \multirow{-2}{*}{Tree-GRPO} & Bamboogle (OOD)& 23.2 $\pm$ 3.8 & 28.8 $\pm$ 4.1  \\
\midrule[\lightrulewidth]
\rowcolor{cellbase} & MuSiQue  (ID) & 4.30 $\pm$ 0.41 & 6.90 $\pm$ 0.52  \\
\rowcolor{cellbase} \multirow{-2}{*}{GRPO + MCTS} & Bamboogle (OOD)& 18.4 $\pm$ 3.5 & 29.6 $\pm$ 4.1  \\
\midrule[\lightrulewidth]
\rowcolor{tabband} & MuSiQue (ID)  & \textbf{9.08} $\pm$ 0.58 & \textbf{10.74} $\pm$ 0.63  \\
\rowcolor{tabband} \multirow{-2}{*}{GRPO + \ours\ (ours)} & Bamboogle (OOD)& \textbf{30.4} $\pm$ 4.1 & \textbf{35.5} $\pm$ 4.3  \\
\bottomrule
\end{tabularx}

\end{table*}

For multi-hop reasoning, we train on the MuSiQue dataset~\citep{musique}, in which the agent must answer complex questions that require retrieving and integrating information across multiple documents. We evaluate on the MuSiQue validation set (in-distribution), and Bamboogle~\citep{bamboogle}, a benchmark of compositional two-hop questions constructed from a different source (out-of-distribution). 

\textbf{Baselines.} We compare against GRPO~\citep{dsr1} and Tree-GRPO~\citep{treegrpo}, with the latter being the current state-of-the-art method for search agent post-training. We apply \ours\ on top of GRPO. Since \ours\ is a sampling algorithm, we additionally compare against GRPO with MCTS sampling, which replaces the rollouts of GRPO with trajectories found by Monte Carlo Tree Search~\citep{rstar,rstarmath}.

\textbf{Training Setup.} We adopt the training setup of Tree-GRPO~\citep{treegrpo}, using Llama-3.2-3B-Instruct and Llama-3.1-8B-Instruct~\citep{llama3} as base models. During the search process, the agent can perform multiple search actions before producing a final answer. Search results are provided by an offline Wikipedia server. We use the solvable subset of the MuSiQue training set as our training data and train for 2 epochs. We evaluate on the full official MuSiQue validation set. Detailed configurations are provided in Appendix~\ref{app:qa_setup}.

As shown in Table~\ref{tab:multihop}, GRPO falls below the base model on both backbones, while Tree-GRPO and MCTS yield limited changes. We attribute the degradation of GRPO to reward sparsity: on MuSiQue, 93\% of GRPO groups receive zero reward for all rollouts, so the group-normalized advantage vanishes and these problems provide no learning signal. The GRPO-trained policy also learns to search less, reducing the number of valid \texttt{<search>} calls per question from 1.62 to 0.84 (Table~\ref{tab:cost_posttraining}), which hurts on questions that require multiple hops. In contrast, GRPO + \ours\ improves MuSiQue accuracy from 4.01\% to 9.08\% (3B) and from 6.58\% to 10.74\% (8B), and Bamboogle accuracy by 6.4 and 13.1 points, respectively. See Section~\ref{sec:cost} for a compute comparison.

\subsubsection{Logical Reasoning}
\label{sec:kk}

\begin{wrapfigure}{r}{0.45\textwidth}
\vspace{-1em}
\centering
\includegraphics[width=0.43\textwidth]{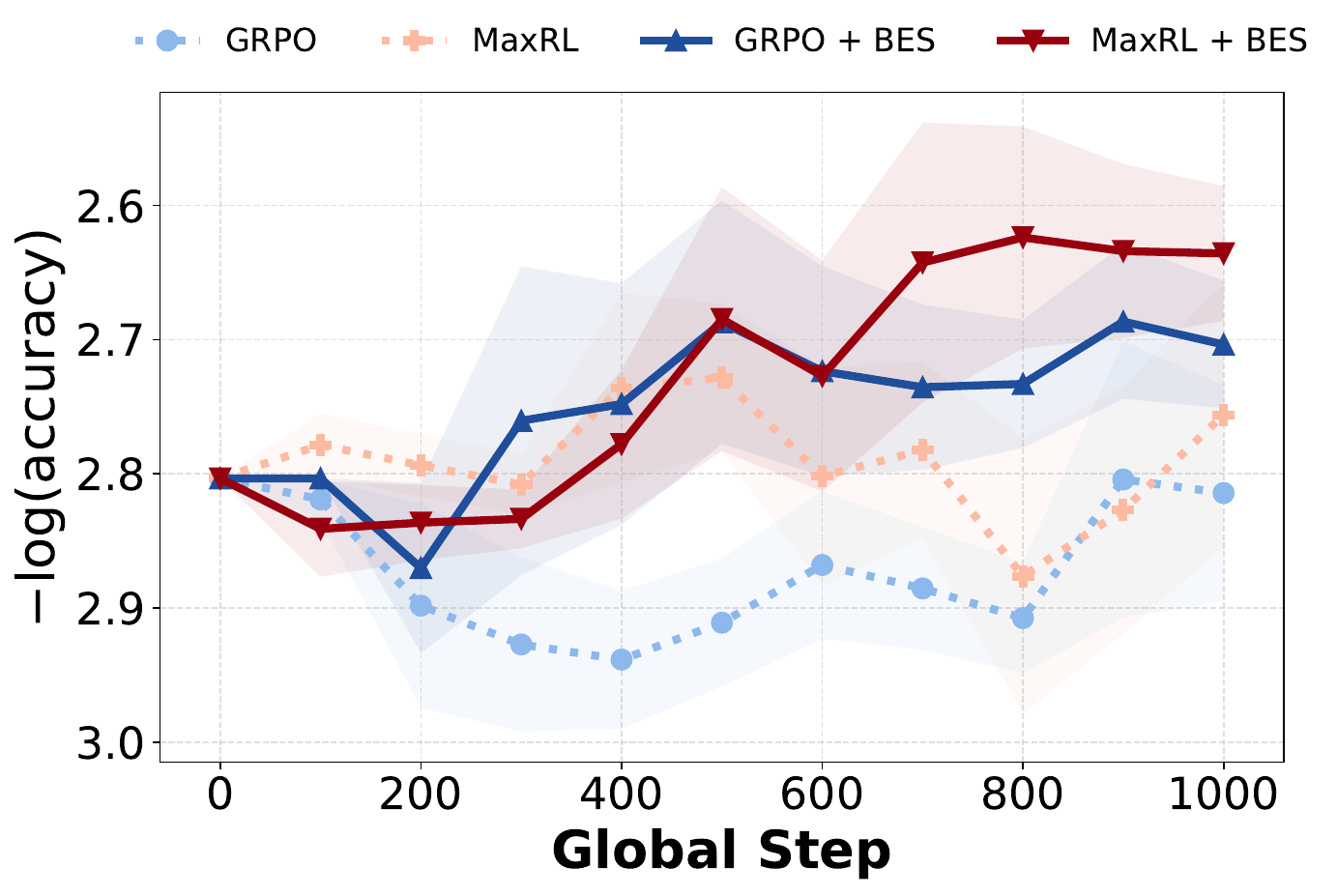}
\caption{EMA-smoothed validation accuracy on logical reasoning (Knights-and-Knaves).}
\label{fig:kk}
\vspace{-1em}
\end{wrapfigure}

For logical reasoning, we use the Knights-and-Knaves dataset~\citep{kk}. In this task, each puzzle involves a group of people who are either a knight or a knave, and the solver must deduce each one's identity.

\textbf{Baselines.} We select two post-training algorithms, GRPO~\citep{dsr1} and MaxRL~\citep{maxrl}, as baselines, as we apply \ours\ on top of both GRPO and MaxRL. We note that \ours\ is a sampling algorithm that is agnostic to the training method and can be applied on top of any post-training algorithm. 

\textbf{Training Setup.} We use Gemma-3-1B-it~\citep{gemma3} as the base model. 
To adapt the model to Knights-and-Knaves, we first perform a cold start by fine-tuning on 1K SFT examples generated using the official data generation pipeline for 3 epochs, followed by 4 epochs of post-training on 5K problems, during which we compare the different methods on a 1.3K validation set. Detailed experimental configurations are provided in Appendix~\ref{app:logical_setup}. 

As shown in Figure~\ref{fig:kk}, due to the difficulty of the training set, GRPO and MaxRL show little to no improvement during training, while \ours\ improves throughout training, demonstrating that \ours\ is more effective at discovering high-quality training samples.

\subsection{Bidirectional Evolutionary Search for Inference}
\label{sec:inference}

\subsubsection{Open Problem Solving}

Here, we consider three benchmarks: Circle Packing (Square), which seeks to pack $N$ circles into a unit square and maximize the sum of radii; Circle Packing (Rect), the analogous problem for a rectangular container; and Heilbronn (Convex), which places $13$ points in a convex region of unit area to maximize the minimum area of any triangle formed by three of the points.

\begin{table*}[h]
\centering
\renewcommand{\arraystretch}{1.20}
\caption{Open problem solving benchmarks with GPT-5 as the backbone model. All open-source frameworks are run under the same backbone model, API budget, and evaluation protocol. We report mean $\pm$ standard deviation and best objective values ($\uparrow$) over 3 runs. \textbf{Bold} denotes the best performing framework; human and AlphaEvolve results are listed for reference only.}
\vspace{6pt}
\resizebox{\textwidth}{!}{%
\begin{tabular}{@{} l cc cc cc @{}}
\toprule
\multirow{2}{*}{\textbf{Strategy}}
& \multicolumn{2}{c}{\textbf{Circle Packing (Square)}} 
& \multicolumn{2}{c}{\textbf{Circle Packing (Rect.)}} 
& \multicolumn{2}{c}{\textbf{Heilbronn (Convex)}} \\
\cmidrule(lr){2-3}
\cmidrule(lr){4-5}
\cmidrule(lr){6-7}
& Avg. & Best 
& Avg. & Best 
& Avg. & Best \\
\midrule
\multicolumn{7}{@{}l}{\textit{Human \& High-compute closed-source framework}} \\
\midrule[\lightrulewidth]
\rowcolor{cellbase} Human & -- & 2.634 & -- & 2.364 & -- & 0.0306 \\
\rowcolor{cellbase}  AlphaEvolve & -- & 2.635 & -- & 2.3658 & -- & 0.0309 \\
\midrule
\multicolumn{7}{@{}l}{\textit{Open-source frameworks}} \\
\midrule[\lightrulewidth]
\rowcolor{cellbase} OpenEvolve & 2.418 $\pm$ .123 & 2.541 & 2.270 $\pm$ .025 & 2.313 & 0.020 $\pm$ .002 & 0.022 \\
\rowcolor{cellbase}  GEPA & 2.537 $\pm$ .004 & 2.541 & 2.274 $\pm$ .046 & 2.351 & 0.025 $\pm$ .001 & 0.026 \\
\rowcolor{cellbase} ShinkaEvolve & 2.478 $\pm$ .076 & 2.565 & 2.332 $\pm$ .023 & 2.357 & 0.023 $\pm$ .001 & 0.024 \\
\midrule
\rowcolor{tabband}  ShinkaEvolve + \ours & \textbf{2.623} $\pm$ .014 & \textbf{2.632} & \textbf{2.349} $\pm$ .012 & \textbf{2.360} & \textbf{0.026} $\pm$ .001 & \textbf{0.027} \\
\bottomrule
\end{tabular}%
}
\label{tab:ops}
\end{table*}

\textbf{Baselines.} We compare against three open-source frameworks: OpenEvolve~\citep{openevolve}, GEPA~\citep{gepa}, and ShinkaEvolve~\citep{shinkaevolve}.  All baseline results use the same backbone model, compute budget, and configuration as ours.  We apply \ours\ on top of ShinkaEvolve. For reference, we also include the results of human experts and AlphaEvolve~\citep{alphaevolve}, which is closed-source and uses significantly more compute than all other frameworks.

\textbf{Setup.} ShinkaEvolve maintains a population of candidate programs and iteratively proposes modifications. Since directly concatenating model outputs is not meaningful on this benchmark (where candidates are executable programs), we implement the evolution operators by prompting LLMs. We use GPT-5 as the backbone model. Detailed configurations are provided in Appendix~\ref{app:inference_setup}. 

As shown in Table~\ref{tab:ops}, \ours\ outperforms all open-source frameworks on average.  Notably, \ours\ also exhibits the lowest run-to-run variance on average across the three benchmarks, indicating that the search is more stable. The advantage of \ours\ also holds with open-weight backbones (Appendix~\ref{app:open_weight}).

\subsubsection{Mathematical Reasoning \& Competitive Programming}

For mathematical reasoning, we use problems of AIME 2024 and AIME 2025~\citep{aime25}. These competition problems require long chains of tightly coupled reasoning steps. For competitive programming, we use LiveCodeBench v6~\citep{livecodebench}, which consists of recent problems from LeetCode, AtCoder, and Codeforces, and we report pass@1 on the hidden test cases.

\textbf{Baselines.} We compare against zero-shot generation, majority voting, MCTS~\citep{mcts} + PRM~\citep{math500}. The PRM is instantiated by prompting the same language model used for answer generation. 

\textbf{Setup.} Since AIME solutions are tightly coupled, the evolution operators act on complete solutions through prompting. The backward search asks the model to write executable self-check predicates for the problem. Detailed configurations are provided in Appendix~\ref{app:math_setup}. 
Similarly, for competitive programming, candidates are complete programs, and all four evolution operators are implemented by prompting. The backward search decomposes the problem into executable program checks. Detailed configurations are provided in Appendix~\ref{app:code_setup}.

As shown in Table~\ref{tab:math_code}, most search methods substantially improve over zero-shot generation. \ours\ achieves the best accuracy on average.

\begin{table*}[t]
\centering
\renewcommand{\arraystretch}{1.15}
\caption{Inference-time results on AIME 2024--2025 and LiveCodeBench v6 with Qwen3 backbones. We report accuracy (\%) as mean $\pm$ standard error over 3 runs. All search-based methods use the same compute budget per problem. \textbf{Bold} denotes the best method per column.}
\vspace{4pt}
\label{tab:math_code}
\resizebox{\textwidth}{!}{%
\begin{tabular}{@{} l ccc ccc @{}}
\toprule
\multirow{2}{*}{\textbf{Method}}
& \multicolumn{3}{c}{\textbf{AIME 2024--2025}}
& \multicolumn{3}{c}{\textbf{LiveCodeBench v6}} \\
\cmidrule(lr){2-4} \cmidrule(lr){5-7}
& Qwen3-4B & Qwen3-8B & Qwen3-14B & Qwen3-4B & Qwen3-8B & Qwen3-14B \\
\midrule
\rowcolor{cellbase} Zero-shot       & 15.6 $\pm$ 1.5 & 13.9 $\pm$ 1.5 & 33.9 $\pm$ 0.6 & 26.3 $\pm$ 0.1 & 24.9 $\pm$ 0.9 & 28.3 $\pm$ 1.4 \\
\rowcolor{cellbase} Majority-of-$N$ & 28.3 $\pm$ 1.7 & 38.3 $\pm$ 1.7 & 46.7 $\pm$ 0.0 & 27.8 $\pm$ 1.2 & 26.3 $\pm$ 0.6 & 29.4 $\pm$ 0.3 \\
\rowcolor{cellbase} MCTS + PRM            & \textbf{30.0} $\pm$  1.0 & 38.3 $\pm$ 2.5 & 43.3 $\pm$ 1.0 & 30.6 $\pm$ 0.2 & 30.3 $\pm$ 0.1 & 34.0 $\pm$ 0.3 \\
\rowcolor{tabband}  \ours\ (ours)   & \textbf{30.0} $\pm$ 1.7 & \textbf{39.4} $\pm$ 0.6 & \textbf{49.4} $\pm$ 0.6 & \textbf{31.9} $\pm$ 0.4 & \textbf{33.1} $\pm$ 0.6 & \textbf{37.1} $\pm$ 1.7 \\
\bottomrule
\end{tabular}%
}
\vspace{-1em}
\end{table*}

\subsection{Ablation Study}

\begin{wrapfigure}{r}{0.45\textwidth}
\vspace{-1em}
\centering
\includegraphics[width=0.43\textwidth]{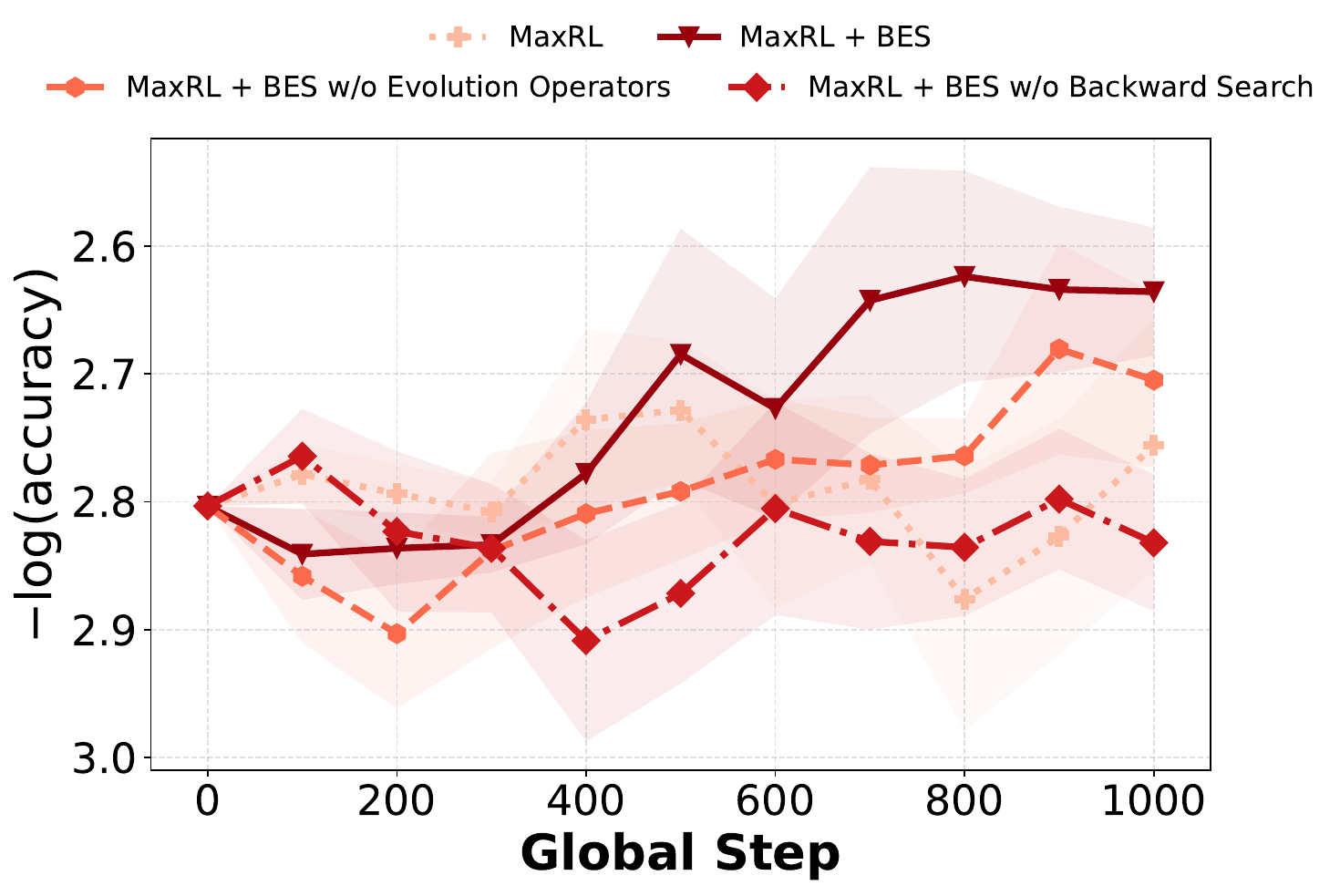}
\vspace{-1em}
\caption{Ablation study on logical reasoning.}
\label{fig:kk_abl}
\vspace{-1em}
\end{wrapfigure}

We conduct an ablation study on the Knights-and-Knaves benchmark. \ours\ consists of two components: the evolution operators in the forward search and the backward search. We therefore consider two ablations: (1) removing the evolution operators; and (2) removing the backward search. As shown in Figure~\ref{fig:kk_abl}, removing the evolution operators degrades performance relative to the full \ours, while removing the backward search performs close to or even worse than the MaxRL baseline. Without the verification provided by the backward search, the forward search reduces to random exploration, which shows that the components of \ours\ are complementary.

\subsection{Cost Analysis}
\label{sec:cost}
To evaluate the cost-effectiveness of \ours, we compare the per-step wall-clock time and accuracy across GRPO, Tree-GRPO, GRPO + MCTS, and GRPO + \ours\ when training Llama-3.2-3B-Instruct on the agentic search task. We report the median wall-clock time per step throughout the training.

\begin{table}[h]
\centering
\renewcommand{\arraystretch}{1.20}
\caption{Cost analysis for multi-hop reasoning post-training on Llama-3.2-3B-Instruct. Accuracy (\%) on MuSiQue and Bamboogle is reported from Table~\ref{tab:multihop}. \# Valid Search is the average number of valid \texttt{<search>} calls per question by the agent. Walltime is the median wall-clock time per training step.}
\resizebox{\textwidth}{!}{%
\begin{tabular}{@{} l cc c c @{}}
\toprule
\textbf{Method} & \textbf{MuSiQue Acc.\ (\%)} & \textbf{Bamboogle Acc.\ (\%)} & \textbf{\# Valid Search} & \textbf{Walltime (s)} \\
\midrule
\rowcolor{cellbase} Base model  & 4.01 $\pm$ 0.40 & 24.0 $\pm$ 3.8 & 1.62 & --  \\
\rowcolor{cellbase} GRPO        & 2.11 $\pm$ 0.29 & 16.8 $\pm$ 3.3 & 0.84 & 64  \\
\rowcolor{cellbase} Tree-GRPO   & 3.93 $\pm$ 0.40 & 23.2 $\pm$ 3.8 & 1.50 & 240 \\
\rowcolor{cellbase} GRPO + MCTS   & 4.30 $\pm$ 0.41 & 18.4 $\pm$ 3.5 & 1.88 & 382 \\
\rowcolor{tabband}  GRPO + \ours\ (ours) & \textbf{9.08} $\pm$ 0.58 & \textbf{30.4} $\pm$ 4.1 & \textbf{2.31} & 309 \\
\bottomrule
\end{tabular}%
}
\label{tab:cost_posttraining}
\end{table}

Search is more expensive than i.i.d.\ rollouts: \ours\ takes about $4.8\times$ the per-step time of GRPO. However, GRPO wastes most of its compute, since 93\% of its groups carry no learning signal. Among the search-based samplers, \ours\ takes 19\% less time per step than GRPO + MCTS while more than doubling its MuSiQue accuracy (9.08\% vs.\ 4.30\%) and improving Bamboogle accuracy by 12.0 points. Compared with Tree-GRPO, \ours\ adds 29\% to the per-step time but improves MuSiQue accuracy by $2.3\times$ and Bamboogle accuracy by 7.2 points. Thus, the gains of \ours\ cannot be attributed to extra compute alone. Instead, they stem from how the search constructs and selects candidates.

\section{Related Work}

\textbf{Self-Improvement for LLM and Agent.} A growing line of work studies how language models and agents can improve using their own outputs to self-evolve~\citep{selfimprovement, selftraining}. STaR~\citep{star} bootstraps reasoning by filtering correct rationales for fine-tuning, while Self-Refine~\citep{madaan2023self} revises outputs at inference time. Self-Rewarding Language Models~\citep{selfreward} use the model itself as a judge for preference optimization. Reflexion~\citep{shinn2023reflexion} converts environmental feedback into verbal reflections, and Voyager~\citep{wang2023voyager} accumulates reusable skills through continual exploration. These methods typically refine individual trajectories or rely on the model to judge its outputs. In contrast, \ours\ treats self-improvement as a search problem, discovering solutions that facilitate self-improvement.

\textbf{Classical Search Methods.} The recent surge of search methods in LLMs draws on a long history of classical search. In graph search, algorithms such as A*~\citep{astar} and bidirectional search introduced principles including heuristic guidance, frontier expansion, and search-space reduction. Branch-and-bound methods~\citep{bbsurvey}, which prune subtrees whose bounds certify suboptimality, offer a natural analogy for methods that use verifiers to discard unpromising candidates. Evolutionary methods, including genetic search~\citep{genesearch,geneprogram} and differential evolution~\citep{differentialevolution}, optimize by iteratively refining candidate populations. 

Additional related work is discussed in Appendix~\ref{app:rel}.
\section{Conclusion}

In this paper, we present \ours, a bidirectional evolutionary search framework that addresses two fundamental limitations of existing search methods for LLMs and agents: sparse verification signals and confined candidate generation. \ours\ couples forward search, which evolves candidates through expansion and evolution operators, with backward search, which recursively decomposes problems into verifiable sub-goals to provide dense feedback. We illustrate that candidates generated by expansion-only search may be confined to a narrow entropy shell while evolution operators may escape it. Experiments demonstrate that \ours\ consistently improves over baselines in both post-training and inference. We hope \ours\ provides useful insights for self-improving language models.

\clearpage

\bibliography{ref}
\bibliographystyle{iclr2027_conference}

\newpage
\appendix

\etocdepthtag.toc{appendix}
\etocsettocstyle{\section*{Appendix - Table of Contents}}{}
\tableofcontents

\clearpage
\section{Additional Related Work}
\label{app:rel}

\textbf{Search in LLM and Agent.} Recent work applies search to both training and inference of LLMs and agents~\citep{efficientevolution, funsearch}. On the training side, tree-structured exploration provides richer supervision and finer-grained credit assignment than standard rollout-based methods. Tree-GRPO~\citep{treegrpo} and TreeRL~\citep{treerl} integrate tree search directly into reinforcement learning, while ReST-MCTS*~\citep{restmcts}, MCTS-DPO~\citep{mctsdpo}, and rStar-Math~\citep{rstarmath} use search to bootstrap higher-quality training data for self-improvement. On the inference side, Tree of Thoughts~\citep{tot}, Graph of Thoughts~\citep{got}, and RAP~\citep{rap} established reasoning as an explicit search problem by expanding chain-of-thought into tree-structured exploration~\citep{abmcts}. Since these methods build each candidate by extending a single lineage, they struggle to explore beyond the policy's own distribution.
A complementary line of work uses LLMs in evolutionary search. Language Model Crossover~\citep{lmx} prompts a model with several parents to produce offspring, EvoPrompt~\citep{evoprompt} and Promptbreeder~\citep{promptbreeder} evolve prompts through LLM-driven crossover, and Mind Evolution~\citep{mindevolution} recombines and refines natural-language solutions at inference time. Evolutionary frameworks such as AlphaEvolve~\citep{alphaevolve}, ShinkaEvolve~\citep{shinkaevolve}, and ThetaEvolve~\citep{thetaevolve} further tackle hard and even open problems~\citep{languageevolutionary}. Besides this, Gödel Agent~\citep{godel} and the Darwin Gödel Machine~\citep{darwin} achieve evolution by letting an agent rewrite its own code or scaffold, i.e., they search over agent designs. However, these methods guide selection with a task-level fitness function and are mainly used at inference time. 
Process reward models (PRMs) provide intermediate feedback by learning to evaluate individual reasoning steps~\citep{math500}; generative PRMs such as ThinkPRM additionally produce reasoning to verify these steps~\citep{process}. In \ours, backward search generates executable sub-goal checks that can be reused across candidates, eliminating the need to train a separate PRM.

\section{Pseudo Code}
\label{app:algo}

We provide the full pseudocode for \ours\ below. Algorithm~\ref{alg:bes} describes the main loop, which alternates between forward search steps (Algorithm~\ref{alg:forward}) and backward scoring (Algorithm~\ref{alg:backward}). Every $K_{\mathrm{dec}}$ steps, the backward goal tree is refined by decomposing an unsolved leaf into finer sub-goals (Algorithm~\ref{alg:decompose}), after which all pool scores are recomputed under the updated tree. When the task provides no ground-truth verifier, the root-goal verifier is  written by the policy.

\begin{algorithm}[h]
\caption{\ours: Bidirectional Evolutionary Search}
\label{alg:bes}
\begin{algorithmic}[1]
\Require problem $x$, policy $\pi_\theta$, verifier $V$, budget $B$, decomposition interval $K_{\mathrm{dec}}$
\State Initialize $\mathcal{G}$ with root goal $g_0$; prompt $\pi_\theta$ to decompose $g_0$ into initial sub-goals
\State $\mathcal{P} \leftarrow \{n_0\}$ with $n_0 = ()$ \Comment{Empty root node}
\For{$u = 0, 1, \ldots$ \textbf{while} $\mathrm{calls}(u) < B$}
    \State $n' \leftarrow \textsc{ForwardStep}(\mathcal{P}, \mathcal{C}_u, \mathcal{G}, \tau_u)$ \Comment{Algorithm~\ref{alg:forward}}
    \State $s(n') \leftarrow \textsc{BackwardScore}(n', g_0, \mathcal{G})$ \Comment{Algorithm~\ref{alg:backward}}
    \State $\mathcal{P} \leftarrow \mathcal{P} \cup \{n'\}$
    \If{$(u+1) \bmod K_{\mathrm{dec}} = 0$}
        \State $\mathcal{G} \leftarrow \textsc{BackwardDecompose}(x, \mathcal{G}, \mathcal{P})$ \Comment{Algorithm~\ref{alg:decompose}}
        \For{$n \in \mathcal{P}$} \Comment{Recompute all scores under refined tree}
            \State $s(n) \leftarrow \textsc{BackwardScore}(n, g_0, \mathcal{G})$
        \EndFor
    \EndIf
\EndFor
\State \Return $\arg\max_{n \in \mathcal{P}:\, \mathrm{terminal}(n)} V(x, n)$
\end{algorithmic}
\end{algorithm}

Algorithm~\ref{alg:forward} details a single forward search step. An operator is sampled from the fixed distribution over the five operators. In \textsc{Apply}, operators are realized either as direct edits or by prompting the policy.

\begin{algorithm}[h]
\caption{\textsc{ForwardStep}: Forward Search Step}
\label{alg:forward}
\begin{algorithmic}[1]
\Require pool $\mathcal{P}$, eligible set $\mathcal{C}_u$, goal tree $\mathcal{G}$, temperature $\tau_u$
\State Sample operator $o \in \{\textit{expand}, \textit{combine}, \textit{delete}, \textit{translocate}, \textit{crossover}\}$
\If{$o \in \{\textit{expand}, \textit{delete}\}$} \Comment{Single-parent operators}
    \State Sample parent $n$ from $\mathcal{C}_u$ via Boltzmann selection (Eq.~\ref{eq:boltz})
    \State $n' \leftarrow \textsc{Apply}(o, n)$
\Else \Comment{Two-parent operators}
    \State Sample pair $(n_a, n_b)$ from $\mathcal{C}_u$ via pair Boltzmann selection (Eq.~\ref{eq:bi_boltz})
    \State $n' \leftarrow \textsc{Apply}(o, n_a, n_b)$
\EndIf
\State \Return $n'$
\end{algorithmic}
\end{algorithm}

Algorithm~\ref{alg:backward} computes the backward score for a candidate node by recursively traversing the sub-goal tree. If a sub-goal is fully satisfied, the score is 1; otherwise, the score blends the local verifier output with the average score over its children.

\begin{algorithm}[h]
\caption{\textsc{BackwardScore}: Backward Search Scoring}
\label{alg:backward}
\begin{algorithmic}[1]
\Require problem $x$, candidate node $n$, sub-goal $g$, goal tree $\mathcal{G}$
\If{$V_g(x, n) = 1$}
    \State \Return $1$
\ElsIf{$\mathrm{ch}(g) \neq \emptyset$}
    \State \Return $\alpha \cdot V_g(x, n) + (1 - \alpha) \cdot \frac{1}{|\mathrm{ch}(g)|} \sum_{g' \in \mathrm{ch}(g)} \textsc{BackwardScore}(n, g', \mathcal{G})$
\Else
    \State \Return $V_g(x, n)$
\EndIf
\end{algorithmic}
\end{algorithm}

Algorithm~\ref{alg:decompose} refines the goal tree by selecting a random unsolved leaf sub-goal and prompting the policy to decompose it into finer children, each equipped with its own local verifier.

\begin{algorithm}[h]
\caption{\textsc{BackwardDecompose}: Goal Tree Decomposition}
\label{alg:decompose}
\begin{algorithmic}[1]
\Require problem $x$, current goal tree $\mathcal{G}$, pool $\mathcal{P}$
\State $\mathcal{L} \leftarrow \{g \in \mathcal{G} : \mathrm{ch}(g) = \emptyset \wedge \max_{n \in \mathcal{P}} V_g(x, n) < 1\}$ \Comment{Unsolved leaves}
\State Sample $g^*$ uniformly from $\mathcal{L}$
\State Prompt $\pi_\theta$ to decompose $g^*$ into children $\{g'_1, \ldots, g'_c\}$ with local verifiers $\{V_{g'_1}, \ldots, V_{g'_c}\}$
\State $\mathrm{ch}(g^*) \leftarrow \{g'_1, \ldots, g'_c\}$; \, add to $\mathcal{G}$
\State \Return $\mathcal{G}$
\end{algorithmic}
\end{algorithm}

\section{Formal Definitions of Evolution Operators}
\label{app:operators}
The definitions below are for the direct-edit instantiation of the evolution operators, in which a candidate is a sequence of self-contained steps with clear boundaries (e.g., paragraph-level reasoning steps in Knights-and-Knaves and action triples in MuSiQue). When direct edits are not meaningful, e.g., for executable programs or tightly coupled derivations, the same operators are instantiated by prompting the policy model.

For two parents $n_a, n_b$, let $s=\max\{\ell:y_{a,1:\ell}=y_{b,1:\ell}\}$ denote their common-prefix length, and let $\sigma_a=(y_{a,s+1},\dots,y_{a,t_a})$ and $\sigma_b=(y_{b,s+1},\dots,y_{b,t_b})$ be their respective suffixes.

\textit{(i) Combination.} $n' = y_{a,1:s}\,\oplus\,\sigma_a\,\oplus\,\sigma_b$.

\textit{(ii) Deletion.} $n' = (y_{1},\dots,y_{\ell-1},\,y_{\ell+1},\dots,y_{t})$, where $\ell\sim\mathrm{Uniform}\{2,\dots,t-1\}$.

\textit{(iii) Translocation.} $n' = y_{a,1:s}\,\oplus\,\sigma_a[1:r-1]\,\oplus\,(\sigma_b)_q\,\oplus\,\sigma_a[r+1:m_a]$, where $r\sim\mathrm{Uniform}\{1,\dots,m_a\}$ and $q\sim\mathrm{Uniform}\{1,\dots,m_b\}$.

\textit{(iv) Crossover.} $n' = y_{a,1:s}\,\oplus\,\sigma_a[1:i]\,\oplus\,\sigma_b[j+1:m_b]$, where $i\in\{0,\dots,m_a\}$ and $j\in\{0,\dots,m_b-1\}$ are sampled uniformly.

\section{Intuition for Evolution Operators}
\label{app:theory-mutation}
In this section, we provide an intuition for the
evolution operators. This is not a formal theory.

Expansion generates a continuation conditioned
on an existing prefix, whereas evolution assembles blocks from
different trajectories. For a fixed partition of trajectory positions,
independently selecting a block from each corresponding library yields
a Cartesian product of block choices.

Fix a task $x$ and a positive integer horizon $T$. We use token-level
steps in the analysis: each step takes values in a finite vocabulary
$\mathcal{V}$. Let $Y=(Y_1,\ldots,Y_T)$ denote a random trajectory and
$y=(y_1,\ldots,y_T)$ a realization. For each position
$t\in\{1,\ldots,T\}$, write $Y_{<t}=(Y_1,\ldots,Y_{t-1})$ for the
random prefix and $y_{<t}$ for its realization; the prefix at $t=1$
is empty. The policy $\pi_\theta$, parameterized by $\theta$, induces
the trajectory distribution
\begin{equation}
  P(y) := \pi_\theta(y\mid x)
  = \prod_{t=1}^{T} P(y_t\mid y_{<t}).
\end{equation}
Here, $P(y_t\mid y_{<t})$ is the policy's next-token probability.
We suppress conditioning on the fixed task $x$ throughout.
All logarithms are natural, so information quantities are measured
in nats.

To describe next-token uncertainty, define the \emph{surprisal}
of the generated token and the \emph{conditional entropy} at a
given prefix:
\begin{equation}
  Z_t := -\log P(Y_t\mid Y_{<t}),
  \qquad
  h_t(y_{<t})
  := H_P(Y_t\mid Y_{<t}=y_{<t}),
\end{equation}
where $H_P$ denotes Shannon entropy under $P$.
Thus, $Z_t$ measures how surprising the sampled token is, whereas
$h_t(y_{<t})$ is its expected surprisal when the prefix is fixed.
We also define the position-wise mean conditional entropy
\begin{equation}
  \bar h_t
  := \mathbb{E}_P\!\left[h_t(Y_{<t})\right],
\end{equation}
where $\mathbb{E}_P$ denotes expectation under the native trajectory
distribution $P$. Unlike $h_t(Y_{<t})$, which varies with the
sampled prefix, $\bar h_t$ is deterministic for a fixed task and
position.

The total surprisal of a trajectory is
\begin{equation}
  S_T := -\log P(Y)
  = \sum_{t=1}^{T} Z_t.
\end{equation}
Its expectation is the trajectory entropy:
\begin{equation}
  H_T := H_P(Y)
  = \mathbb{E}_P[S_T]
  = \sum_{t=1}^{T}
      \mathbb{E}_P\!\left[h_t(Y_{<t})\right]
  = \sum_{t=1}^{T}\bar h_t,
\end{equation}
where the decomposition follows from the entropy chain rule.
In particular, $H_T$ is a fixed center determined by the task,
policy, and horizon.

Finally, let
$\mathcal{F}_t=\sigma(Y_1,\ldots,Y_t)$ be the natural filtration,
where $\sigma(\cdot)$ denotes the sigma-algebra generated by the
listed random variables. Thus, $\mathcal{F}_t$ represents the
information available after the first $t$ tokens, and
$\mathcal{F}_0$ is the trivial sigma-algebra. By construction,
\begin{equation}
  \mathbb{E}_P[Z_t\mid\mathcal{F}_{t-1}]
  = h_t(Y_{<t}).
\end{equation}
Consequently, $Z_t-h_t(Y_{<t})$ is a martingale difference.
This identity separates token-sampling fluctuations from
prefix-dependent fluctuations in conditional entropy and forms
the basis of the concentration analysis below.

\subsection{Discussion of Assumptions}
\label{app:assumptions}

We make two assumptions about the process.

\begin{assumption}[Local context correlation]
\label{assump:correlation}
There exists a constant $\rho>0$ such that, for every
$t\in\{2,\ldots,T\}$,
\begin{equation}
  I_P(Y_t;Y_{t-1}) \geq \rho,
\end{equation}
where $I_P$ denotes mutual information under $P$.
\end{assumption}

This requires the preceding token to provide predictive information
about the current token: observing $Y_{t-1}$ reduces the uncertainty
about $Y_t$ by at least $\rho$ nats on average. It captures the local
sequential structure of language and reasoning, where nearby tokens
are connected through grammatical constraints, semantic relations,
and the progression of an argument. The constant $\rho$ quantifies
the minimum strength of this local dependence.

\begin{assumption}[Limited next-token uncertainty]
\label{assump:limited-uncertainty}
There exists a constant $\kappa\geq0$ such that, for every
$t\in\{2,\ldots,T\}$ and every prefix $y_{<t}$ with
$P(y_{<t})>0$,
\begin{equation}
  h_t(y_{<t}) \leq \kappa.
\end{equation}
\end{assumption}

This requires the policy's next-token uncertainty to remain bounded
after observing the full prefix. It formalizes the intuition that
context narrows the range of plausible continuations, with $\kappa$
controlling the remaining uncertainty. Different prefixes may favor
entirely different tokens and induce different probability
distributions while satisfying the same entropy bound. Together,
the two assumptions describe a generation process in which context
informs the prediction and the uncertainty remaining after
conditioning on that context is limited.

\subsection{Shell Confinement of Expansion}

The deviation $S_T-H_T$ between a trajectory's actual information
content and its expected entropy governs how tightly ordinary
rollouts cluster around the typical set. We decompose this
deviation into token-sampling fluctuations and conditional-entropy
drift, and bound each separately.

\begin{lemma}[Martingale decomposition]
\label{lem:decomp}
Define
\begin{equation}
  D_t:=Z_t-h_t(Y_{<t}),
  \qquad
  M_T:=\sum_{t=1}^{T}D_t,
  \qquad
  R_T:=\sum_{t=1}^{T}
    \bigl(h_t(Y_{<t})-\bar h_t\bigr),
\end{equation}
where $\bar h_t=\mathbb{E}_P[h_t(Y_{<t})]$.
Then $(D_t,\mathcal{F}_t)$ is a martingale difference sequence,
and
\begin{equation}
  S_T-H_T=M_T+R_T.
  \label{eq:martingale-decomp}
\end{equation}
Here $M_T$ captures token-sampling fluctuations, while $R_T$
captures the accumulated deviation of conditional entropy
from its position-wise mean.
\end{lemma}

\begin{proof}[Derivation]
By definition,
\[
  \mathbb{E}_P[Z_t\mid\mathcal{F}_{t-1}]
  =h_t(Y_{<t}),
\]
so $\mathbb{E}_P[D_t\mid\mathcal{F}_{t-1}]=0$.
Moreover,
\[
  S_T=\sum_{t=1}^{T}D_t+\sum_{t=1}^{T}h_t(Y_{<t}),
  \qquad
  H_T=\sum_{t=1}^{T}\bar h_t.
\]
Subtracting these identities gives
Eq.~\eqref{eq:martingale-decomp}.
\end{proof}

\begin{lemma}[Bound on entropy drift]
\label{lem:rt}
Under Assumption~\ref{assump:limited-uncertainty},
\begin{equation}
  |R_T|\leq (T-1)\kappa
  \qquad P\text{-almost surely}.
  \label{eq:rt-bound}
\end{equation}
\end{lemma}

\begin{proof}[Derivation]
The first-step prefix is fixed, so
$h_1(Y_{<1})=\bar h_1$.
For every $t\geq2$,
Assumption~\ref{assump:limited-uncertainty} gives
\[
  0\leq h_t(Y_{<t})\leq\kappa,
  \qquad
  0\leq\bar h_t\leq\kappa.
\]
Consequently,
\[
  |R_T|
  \leq\sum_{t=2}^{T}
    |h_t(Y_{<t})-\bar h_t|
  \leq(T-1)\kappa.
\]
\end{proof}

\begin{lemma}[Concentration of $M_T$]
\label{lem:mt}
Let
\begin{equation}
  L:=1+\log|\mathcal{V}|.
\end{equation}
Then, for every $u>0$,
\begin{equation}
  \Pr(|M_T|\geq u)
  \leq
  2\exp\!\left(
    -\frac{u^2}{4TL^2+2Lu}
  \right).
  \label{eq:mt-freedman}
\end{equation}
\end{lemma}

\begin{proof}[Derivation]
For every $r\geq0$, the finite vocabulary gives
\begin{equation}
  \Pr\!\left(
    Z_t\geq\log|\mathcal{V}|+r
    \,\middle|\,\mathcal{F}_{t-1}
  \right)
  \leq e^{-r}.
\end{equation}
Indeed, this event contains at most $|\mathcal{V}|$ tokens,
each with conditional probability at most
$e^{-r}/|\mathcal{V}|$.

Since $0\leq h_t(Y_{<t})\leq\log|\mathcal{V}|$, we have
\[
  |D_t|\leq\max\{Z_t,\log|\mathcal{V}|\}.
\]
The same tail bound therefore holds with $Z_t$ replaced by
$|D_t|$. Integrating this bound gives, for every integer
$q\geq2$,
\begin{equation}
  \mathbb{E}_P[|D_t|^q\mid\mathcal{F}_{t-1}]
  \leq
  q!\sum_{j=0}^{q}
    \frac{(\log|\mathcal{V}|)^j}{j!}
  \leq q!L^q.
\end{equation}
Using $\mathbb{E}_P[D_t\mid\mathcal{F}_{t-1}]=0$, for every
real $\lambda$ with $|\lambda|<1/L$,
\begin{align}
  \mathbb{E}_P[e^{\lambda D_t}\mid\mathcal{F}_{t-1}]
  &\leq
  1+\sum_{q=2}^{\infty}(|\lambda|L)^q \\
  &\leq
  \exp\!\left(
    \frac{\lambda^2L^2}{1-|\lambda|L}
  \right).
\end{align}
Iterating conditional expectations yields
\begin{equation}
  \mathbb{E}_P[e^{\lambda M_T}]
  \leq
  \exp\!\left(
    \frac{T\lambda^2L^2}{1-|\lambda|L}
  \right).
\end{equation}
Applying the exponential Markov inequality with
\[
  \lambda=\frac{u}{2TL^2+Lu}
\]
gives
\[
  \Pr(M_T\geq u)
  \leq
  \exp\!\left(
    -\frac{u^2}{4TL^2+2Lu}
  \right).
\]
Applying the same argument to $-M_T$ and combining the
two tails gives Eq.~\eqref{eq:mt-freedman}.
\end{proof}

Combining the entropy-drift bound with the martingale
concentration bound yields exponential confinement around
the fixed trajectory entropy $H_T$.

\begin{lemma}[Finite-sample concentration for policy rollouts]
\label{thm:aep}
Under Assumption~\ref{assump:limited-uncertainty}, for every
$u>(T-1)\kappa$,
\begin{equation}
  \Pr(|S_T-H_T|\geq u)
  \leq
  2\exp\!\left(
    -\frac{\bigl(u-(T-1)\kappa\bigr)^2}
    {4TL^2+2L\bigl(u-(T-1)\kappa\bigr)}
  \right),
  \label{eq:aep}
\end{equation}
where $L$ is defined in Observation~\ref{lem:mt}.
In particular, for a fixed vocabulary, fixed $\kappa$, and
every fixed $\epsilon>\kappa$,
\begin{equation}
  \Pr\!\left(
    \left|\frac{S_T}{T}-\frac{H_T}{T}\right|
    \geq\epsilon
  \right)
  \leq \exp(-\Omega(T)).
\end{equation}
\end{lemma}

\begin{proof}[Derivation]
By Observation~\ref{lem:decomp},
$S_T-H_T=M_T+R_T$.
Observation~\ref{lem:rt} therefore gives
\[
  \Pr(|S_T-H_T|\geq u)
  \leq
  \Pr\!\left(
    |M_T|\geq u-(T-1)\kappa
  \right).
\]
Applying Observation~\ref{lem:mt} gives Eq.~\eqref{eq:aep}.
For $u=\epsilon T$ with fixed $\epsilon>\kappa$,
\[
  u-(T-1)\kappa
  =(\epsilon-\kappa)T+\kappa.
\]
The numerator in the exponent is then of order $T^2$,
while the denominator is of order $T$, giving the
exponential concentration rate.
\end{proof}

\begin{corollary}[Shell confinement]
\label{cor:shell}
For a fixed $\epsilon>\kappa$, define the typical set
\begin{equation}
  A_\epsilon^{(T)}
  :=\{y:|-\log P(y)-H_T|\leq\epsilon T\}.
\end{equation}
Under Observation~\ref{thm:aep},
\[
  P(A_\epsilon^{(T)})
  \geq1-\exp(-\Omega(T)).
\]
Moreover,
\begin{equation}
  |A_\epsilon^{(T)}|
  \leq\exp(H_T+\epsilon T).
  \label{eq:typical-upper}
\end{equation}
Thus, expansion through native policy rollouts remains
within this shell with exponentially high probability.
\end{corollary}

\begin{proof}[Derivation]
The probability statement follows from
Observation~\ref{thm:aep}.
Every $y\in A_\epsilon^{(T)}$ satisfies
\[
  P(y)\geq\exp\bigl(-(H_T+\epsilon T)\bigr).
\]
Summing probabilities over the typical set gives
\[
  1\geq
  |A_\epsilon^{(T)}|
  \exp\bigl(-(H_T+\epsilon T)\bigr),
\]
which gives Eq.~\eqref{eq:typical-upper}.
\end{proof}

\subsection{Shell Escape via Evolution}

We first relate the four evolution operators to a common
recombination model. The reductions use independent
native policy rollouts as donors, fixed edit locations, and
transferred tokens placed at their original positions.
We use the convention $-\log0=+\infty$.

Partition the coordinates of a trajectory $Y$ into $k\geq2$
nonempty, disjoint groups $U_1,\ldots,U_k$ that together cover
all $T$ positions. Each group records its tokens in their
original order and may contain several separated fragments.
Let $P_j$ be the native marginal distribution of $U_j$.
A $k$-way evolution takes group $j$ from an independent donor
and returns its tokens to their assigned positions. Identifying
the grouped coordinates with the full trajectory, its law is
\begin{equation}
  Q=\bigotimes_{j=1}^{k}P_j.
  \label{eq:product-law}
\end{equation}
Thus, $k$ counts independent donor groups.

\textit{(i) Combination.}
Combination appends $\sigma_b$ to
$y_{a,1:s}\oplus\sigma_a$, giving a concatenation of two
parent fragments. This moves the appended
continuation to new positions and changes the trajectory
length. We use combination as an intuitive extension of the
recombination mechanism and do not establish an escape
guarantee for its general variable-length form.

\textit{(ii) Deletion.}
Deletion removes a paragraph and reconnects the retained fragments.  It motivates exploration through changes to the context, but is not included in the
illustration.

\textit{(iii) Translocation.}
For translocation at a fixed position
$t\in\{2,\ldots,T\}$, retain all coordinates except $Y_t$
from one donor and replace $Y_t$ with the same-position
token from another independent donor. Writing $Y_{-t}$
for the subvector containing all coordinates except $t$,
this gives
\begin{equation}
  U_1=Y_{-t},
  \qquad U_2=Y_t,
  \qquad
  Q=P_{Y_{-t}}\otimes P_{Y_t}.
\end{equation}
The retained left and right fragments together form one
group. Consequently, this is a $k=2$ construction, with
\begin{equation}
  H(Q)-H_T
  =I_P(Y_t;Y_{-t})
  \geq\rho.
\end{equation}
Replacing $k-1$ distinct positions using fresh independent
donors similarly produces $k$ groups: one retained group
and $k-1$ singleton groups. Transfers between different
positions are treated as intuitive extensions here.

\textit{(iv) Crossover.}
At a fixed, aligned cut $s\in\{1,\ldots,T-1\}$, crossover
takes $Y_{1:s}$ from one donor and $Y_{s+1:T}$ from another.
It therefore gives
\begin{equation}
  U_1=Y_{1:s},
  \qquad U_2=Y_{s+1:T},
  \qquad
  Q=P_{Y_{1:s}}\otimes P_{Y_{s+1:T}},
\end{equation}
with
\begin{equation}
  H(Q)-H_T
  =I_P(Y_{1:s};Y_{s+1:T})
  \geq\rho.
\end{equation}
Successive crossovers at fixed boundaries
$0=s_0<s_1<\cdots<s_k=T$, each introducing a fresh
independent donor, yield the contiguous groups
$U_j=Y_{s_{j-1}+1:s_j}$ and hence the $k$-way product law
in Eq.~\eqref{eq:product-law}.

We now illustrate these information gains and the resulting shell escape, beginning with a single crossover and then extending to independent donor groups.

\begin{lemma}[Splice-point KL identity]
\label{lem:splice}
Split a trajectory at position $s\in\{1,\ldots,T-1\}$ and write
$V=Y_{1:s}$ and $U=Y_{s+1:T}$. Let $(V,U)$ and $(V',U')$
be two independent policy rollouts, and form the crossover
trajectory $\widetilde{Y}=(V,U')$. The expected increase in
native surprisal is
\begin{equation}
  \mathbb{E}[-\log P(V,U')]-H_T
  =
  \mathbb{E}_{V,V'}
  \left[
    D_{\mathrm{KL}}
    \bigl(P_{U\mid V'}\|P_{U\mid V}\bigr)
  \right].
  \label{eq:splice-kl}
\end{equation}
Equivalently,
\begin{equation}
  \mathbb{E}[-\log P(V,U')]-H_T
  =
  I_P(V;U)
  +D_{\mathrm{KL}}
    \bigl(P_V\otimes P_U\|P_{V,U}\bigr)
  \geq I_P(V;U).
  \label{eq:splice-mi}
\end{equation}
The gap is strictly positive whenever the suffix distribution
depends on the prefix. Under Assumption~\ref{assump:correlation},
the gap is at least $\rho$.
\end{lemma}

\begin{proof}[Derivation]
Using $P(V,U)=P_V(V)P_{U\mid V}(U)$,
\begin{align}
  \mathbb{E}[-\log P(V,U')]
  &=
  H_P(V)
  +\mathbb{E}_{V,V'}
   \mathbb{E}_{U\sim P_{U\mid V'}}
   [-\log P_{U\mid V}(U)],\\
  H_T
  &=
  H_P(V)+\mathbb{E}_{V'}H(P_{U\mid V'}).
\end{align}
Subtracting gives Eq.~\eqref{eq:splice-kl}.

The marginal law of $(V,U')$ is $P_V\otimes P_U$, so its
expected native surprisal is
\[
  H(P_V\otimes P_U)
  +D_{\mathrm{KL}}
    \bigl(P_V\otimes P_U\|P_{V,U}\bigr).
\]
Since
\[
  H(P_V\otimes P_U)-H_T=I_P(V;U),
\]
Eq.~\eqref{eq:splice-mi} follows. Finally, data processing
and Assumption~\ref{assump:correlation} give
\[
  I_P(V;U)
  \geq I_P(Y_s;Y_{s+1})
  \geq\rho.
\]
\end{proof}

Observation~\ref{lem:splice} analyzes crossover at a single splice
point. We now generalize to $k$-way block evolution, where
each block is drawn from an independent seed trajectory.

\begin{lemma}[Block evolution cross-entropy gap]
\label{lem:tc}
For $\widetilde{Y}\sim Q$,
\begin{equation}
  \mathbb{E}_Q[-\log P(\widetilde{Y})]-H_T
  =
  \mathrm{TC}_P(U_1,\ldots,U_k)
  +D_{\mathrm{KL}}(Q\|P)
  \geq\mathrm{TC}_P(U_1,\ldots,U_k),
  \label{eq:tc-gap}
\end{equation}
where
\[
  \mathrm{TC}_P(U_1,\ldots,U_k)
  :=
  \sum_{j=1}^{k}H_P(U_j)-H_T
\]
is the total correlation between the donor groups.
Under Assumption~\ref{assump:correlation},
\begin{equation}
  \mathrm{TC}_P(U_1,\ldots,U_k)\geq(k-1)\rho.
  \label{eq:tc-local}
\end{equation}
\end{lemma}

\begin{proof}[Derivation]
Since $Q$ is the product of the native group marginals,
\[
  H(Q)=\sum_{j=1}^{k}H_P(U_j).
\]
The cross-entropy identity
\[
  \mathbb{E}_Q[-\log P(\widetilde{Y})]
  =H(Q)+D_{\mathrm{KL}}(Q\|P)
\]
therefore gives Eq.~\eqref{eq:tc-gap}.

To derive Eq.~\eqref{eq:tc-local}, order the groups by their
first token positions. Let $t_j$ be the first position
belonging to $U_j$, so
\[
  1=t_1<t_2<\cdots<t_k\leq T.
\]
For every $j\geq2$, position $t_j-1$ belongs to an earlier
group, because its group's first position is smaller
than $t_j$. Thus, writing
$U_{<j}=(U_1,\ldots,U_{j-1})$, data processing gives
\[
  I_P(U_j;U_{<j})
  \geq I_P(Y_{t_j};Y_{t_j-1})
  \geq\rho.
\]
The entropy chain rule then yields
\begin{equation}
  \mathrm{TC}_P(U_1,\ldots,U_k)
  =
  \sum_{j=2}^{k}I_P(U_j;U_{<j})
  \geq(k-1)\rho.
\end{equation}
\end{proof}

Putting these together gives the following illustration.

\noindent\textbf{Motivation~\ref{thm:shell}}
(Shell confinement and escape)\textbf{.}
\medskip

\noindent
Under Assumptions~\ref{assump:correlation}
and~\ref{assump:limited-uncertainty}, fix $\epsilon>\kappa$
and define the shell
\[
  A_\epsilon^{(T)}
  :=\{y:|-\log P(y)-H_T|\leq\epsilon T\}.
\]

\textbf{(a) Generation and tree search stay within the shell.}
For a trajectory $Y\sim P$ generated by the policy,
\[
  \Pr[Y\notin A_\epsilon^{(T)}]
  \leq\exp(-\Omega(T)).
\]
This shell contains at most
$\exp(H_T+\epsilon T)$ trajectories. 

\textbf{(b) Evolution Operators can escape the shell.}
Construct $\widetilde{Y}$ by cutting independent policy
rollouts at the same $m$ fixed positions and joining the
resulting $m+1$ segments, taking each segment from a
different rollout at its original positions.
The $m$ joins yield an expected surprisal increase of
at least $m\rho$:
\begin{equation}
  \mathbb{E}[-\log P(\widetilde{Y})]
  \geq H_T+m\rho.
\end{equation}
Whenever $m\rho>\epsilon T+\log2$, the expected surprisal
exceeds the upper shell boundary, and
\begin{equation}
  \Pr[\widetilde{Y}\notin A_\epsilon^{(T)}]
  \geq
  \frac{m\rho-\epsilon T-\log2}
       {T\log|\mathcal{V}|-H_T-\epsilon T}
  >0.
\end{equation}
If $m\rho/T$ exceeds $\epsilon$ by a fixed positive margin
as $T$ grows, the recombined candidate escapes with
probability bounded below by a positive constant, whereas
a native rollout escapes only with exponentially small
probability.

\begin{proof}[Derivation]
Part~(a) follows from Observation~\ref{cor:shell}, since
$\epsilon>\kappa$.

For part~(b), let $U_1,\ldots,U_{m+1}$ denote the segments
of a native trajectory at the chosen cut positions.
The recombined trajectory $\widetilde{Y}$ contains
independent segments with these same marginal distributions.
Hence, by Observation~\ref{lem:tc},
\begin{equation}
  H(\widetilde{Y})
  =\sum_{j=1}^{m+1}H_P(U_j)
  \geq H_T+m\rho,
\end{equation}
where $H(\widetilde{Y})$ denotes the entropy of the
recombined candidate. The cross-entropy inequality gives
\[
  \mathbb{E}[-\log P(\widetilde{Y})]
  \geq H(\widetilde{Y})
  \geq H_T+m\rho.
\]

Let
\[
  p:=\Pr[\widetilde{Y}\notin A_\epsilon^{(T)}].
\]
Splitting the entropy according to shell membership
contributes at most $\log2$ nats for the membership
indicator. Inside the shell, the conditional entropy
is at most
$\log|A_\epsilon^{(T)}|\leq H_T+\epsilon T$.
Outside the shell, it is at most $T\log|\mathcal{V}|$.
Therefore,
\begin{equation}
  H_T+m\rho
  \leq H(\widetilde{Y})
  \leq
  \log2+(1-p)(H_T+\epsilon T)
  +pT\log|\mathcal{V}|.
\end{equation}
Rearranging yields
\[
  p\bigl(T\log|\mathcal{V}|-H_T-\epsilon T\bigr)
  \geq m\rho-\epsilon T-\log2.
\]
Moreover, since $H(\widetilde{Y})\leq T\log|\mathcal{V}|$,
\[
  T\log|\mathcal{V}|-H_T-\epsilon T
  \geq m\rho-\epsilon T>0.
\]
Dividing by this positive quantity gives
\[
  \Pr[\widetilde{Y}\notin A_\epsilon^{(T)}]
  \geq
  \frac{m\rho-\epsilon T-\log2}
       {T\log|\mathcal{V}|-H_T-\epsilon T},
\]
which is strictly positive whenever
$m\rho>\epsilon T+\log2$.

Finally, the denominator is at most $T\log|\mathcal{V}|$,
so under the same condition,
\[
  \Pr[\widetilde{Y}\notin A_\epsilon^{(T)}]
  \geq
  \frac{m\rho/T-\epsilon-(\log2)/T}
       {\log|\mathcal{V}|}.
\]
If $m\rho/T-\epsilon$ remains bounded below by a positive
constant as $T$ grows, this lower bound remains positive
for all sufficiently large $T$. Part~(a) gives the
corresponding exponentially small escape probability
for native policy rollouts.
\end{proof}

\emph{Remarks.} A candidate with high surprise under the policy is not necessarily correct. Identifying the useful candidates among them is the role of the backward search, which is why the two components are coupled; without it, the forward search reduces to random exploration.

\section{Detailed Experimental Setup}
\label{app:setup}

\subsection{Multi-Hop Reasoning}
\label{app:qa_setup}

The training set is the
solvable split of the MuSiQue training data; the validation set is the official MuSiQue validation set. The agent emits actions as
\verb|<think>...</think>| followed by either \verb|<search>q</search>|
or \verb|<answer>...</answer>|. Each \verb|<search>| call goes to the offline
retriever which returns the top $3$ passages back to the agent inside
\verb|<information>...</information>| tokens. The trainer uses 2 H200 GPUs. The retriever
server (E5 encoder + FAISS over the 2018 Wikipedia dump) runs independently. 
We conduct two epochs of post-training. Tokens inside \texttt{<information>} are masked from the loss. Trajectories returned by \ours\ are used as the rollouts for each training problem, and their log-probabilities are computed under the original prompt $x$ by a forward pass of the policy.

\paragraph{Forward search.}
The forward search maintains a candidate pool of trajectories per
question, where every trajectory is a sequence of
$(\langle\texttt{think}\rangle,\langle\texttt{search}\rangle,
\langle\texttt{information}\rangle)$ triples optionally followed by a
terminal $(\langle\texttt{think}\rangle,\langle\texttt{answer}\rangle)$
pair.
The evolution operators here operate at the triple boundary. The Boltzmann temperature $\tau$ is annealed linearly from $1.5$ to $0.3$. We pick $\alpha=0.7$.

\paragraph{Backward search.}
During backward search, the policy model decomposes each question into a list of sub-questions.
The original question together with these sub-questions defines the
per-question goal tree.
The local verifier $V_{g_i}$ is queried by the forward search for
every candidate trajectory at every step, so a per-call LLM verifier would be impractical. We therefore
instantiate $V_{g_i}$ as a fast embedding model: every
$\langle\texttt{search}\rangle$ query that a trajectory $n$ has emitted
is embedded into the $\texttt{all-MiniLM-L6-v2}$ sentence-embedding
space, and the same is done for each sub-question. A sub-question
$g_i$ is declared \emph{covered} by trajectory $n$ when
\begin{equation}
\max_{q \in \mathrm{searches}(n)}\,
  \cos\bigl(\mathrm{emb}(q),\mathrm{emb}(g_i)\bigr) \;\ge\; 0.6,
\end{equation}
i.e.\ when at least one of the trajectory's search queries
semantically aligns with the sub-question. 
The reliability of the local verifiers is quantified in Appendix~\ref{app:verifier_quality}.

Hyperparameters are listed in Table \ref{tab:hyperparams_multihop}.

\begin{table*}[h]
\centering
\caption{Hyperparameters for the multi-hop reasoning experiment (MuSiQue).}
\label{tab:hyperparams_multihop}
\begin{tabularx}{\textwidth}{@{}lX@{}}
\toprule
\textbf{Category} & \textbf{Hyperparameter $=$ Value} \\
\midrule
\textbf{Optimization} &
epochs $=$ 2; optimizer $=$ AdamW; learning\_rate $=$ $1\!\times\!10^{-6}$; \\
& lr\_warmup\_ratio $=$ 0.285; group\_size $=$ 8; max\_turns $=$ 3; \\
& train\_batch\_size $=$ 128; ppo\_mini\_batch\_size $=$ 16; ppo\_micro\_batch\_size $=$ 8; \\
& kl\_loss\_coef $=$ $1\!\times\!10^{-3}$; kl\_loss\_type $=$ low\_var\_kl. \\
\midrule
\textbf{Generation} &
max\_prompt\_length $=$ 4096; max\_response\_length $=$ 2048; \\
& temperature $=$ 1.0. \\
\midrule
\textbf{BES search} 
& action\_types $=$ \{expand, delete, combine, translocate, crossover\} \\
& with probs $\{0.70, 0.05, 0.10, 0.075, 0.075\}$. \\
\bottomrule
\end{tabularx}
\end{table*}

\subsection{Logical Reasoning}
\label{app:logical_setup}

We generate the K\&K corpus with the official sampling pipeline. The SFT
cold-start corpus contains $1{,}000$ problems.
The post-training corpus contains $5{,}000$ problems.
Validation contains $1{,}287$ problems.
We conduct a 3-epoch SFT to teach the
output format, followed by a 4-epoch post-training. 
Trajectories returned by \ours\ are used as the rollouts for each training problem, and their log-probabilities are computed under the original prompt $x$ by a forward pass of the policy.

\paragraph{Forward search.}
The forward search maintains a pool of partial reasoning trajectories
partitioned at paragraph granularity (\verb|\n\n|-separated reasoning
steps). 
The Boltzmann temperature $\tau$ is annealed linearly from
$2.0$ to $1.0$.
A trajectory becomes \emph{terminal} when it contains a
``\verb|### Final Answer|'' marker followed by a parseable JSON map
\verb|{name: 0/1}|; the rule-based KK scorer then assigns reward
$r\in\{0,1\}$ from exact match against the ground-truth assignment.
We pick $\alpha=0.3$.

\paragraph{Backward search.}
At the start of every \ours\ search on a puzzle, we instantiate a goal
tree whose root goal is to identify every person's role correctly. 
The rest of the tree is produced through decomposition: given a goal, the model is asked to split it into finer sub-goals. 
As guidance, the prompt suggests that the root can be decomposed into one sub-goal per person, and that each
per-person sub-goal can be further decomposed into elementary
verification strategies, e.g.\ assume the opposite role and seek a contradiction.
Every ten search steps, an unsolved leaf sub-goal is selected and the model is
prompted to decompose it further. Each generated sub-goal is equipped
with a lightweight syntactic verifier produced by the model.

Hyperparameters are listed in Table \ref{tab:hyperparams_logical}.

\begin{table*}[h]
\centering
\caption{Hyperparameters for the logical reasoning experiment (Knights-and-Knaves).}
\label{tab:hyperparams_logical}
\begin{tabularx}{\textwidth}{@{}lX@{}}
\toprule
\textbf{Category} & \textbf{Hyperparameter $=$ Value} \\
\midrule
\textbf{Optimization} &
optimizer $=$ AdamW; learning\_rate $=$ $1\!\times\!10^{-6}$; group\_size $=$ 8;  \\
& train\_batch\_size $=$ 32; ppo\_mini\_batch\_size $=$ 32; \\
& ppo\_micro\_batch\_size\_per\_gpu $=$ 16; ppo\_epochs $=$ 1; \\
& clip\_ratio $=$ 0.2; grad\_clip $=$ 0.3; kl\_coef $=$ 0.0. \\
\midrule
\textbf{Generation} &
max\_prompt\_length $=$ 1024; max\_response\_length $=$ 4096; temperature $=$ 1.0. \\
\midrule
\textbf{BES search} 
& action\_types $=$ \{expand, delete, combine, translocate, crossover\} \\
& with probs $\{0.70, 0.05, 0.10, 0.075, 0.075\}$. \\
\bottomrule
\end{tabularx}
\end{table*}

\subsection{Open Problem Solving}
\label{app:inference_setup}

We choose three open optimization tasks: (i) \emph{Circle Packing (Square)}: pack $n=26$ non-overlapping circles in the unit square to maximize the sum of radii; (ii) \emph{Circle Packing (Rect)}: pack $n=21$ non-overlapping circles in a rectangle of perimeter $4$ with free aspect ratio to maximize the sum of radii; (iii) \emph{Heilbronn (Convex)}: place $13$ points in a convex region of unit area, i.e.\ the minimum triangle area is normalized by the area of the points' convex hull, to maximize the minimum area over all $\binom{13}{3}$ triangles. We report mean and best objective value across $3$ runs per benchmark. We use \texttt{gpt-5} with \texttt{reasoning\_effort = high}. Each run is given 100 generations and \$50 of GPT-5 API usage; both budgets include the calls used for backward decomposition.

\paragraph{Forward search.}
We adopt ShinkaEvolve as the base program-evolution framework.
The four evolution operators are realized at the
program level. We pass both parents to the proposer with an operator-specific instruction (e.g.\ ``combine the
structural decisions of program A with the parameter choices of
program B'') and ask it to return a single new program implementing the requested edit. We pick $\alpha=0.3$.

\paragraph{Backward search.}
In the backward goal tree, each leaf carries a Python verifier
expression that returns a partial-progress score in
$[0,1]$.
Tree growth is performed adaptively during the run. A decomposition pass is triggered whenever the population has failed to improve on the raw objective for five
consecutive generations. We pick a leaf that no program in the archive has yet fully satisfied,
and ask the model to propose sub-goals that together cover the leaf, each with
its own natural-language description and Python verifier.

Hyperparameters are listed in Table \ref{tab:hyperparams_inference}.

\begin{table*}[h]
\centering
\caption{Hyperparameters for the open problem solving experiment.}
\label{tab:hyperparams_inference}
\begin{tabularx}{\textwidth}{@{}lX@{}}
\toprule
\textbf{Category} & \textbf{Hyperparameter $=$ Value} \\
\midrule
\textbf{Generation} &
backbone $=$ gpt-5; reasoning\_effort $=$ high. \\
\midrule
\textbf{BES Search} 
& action\_types $=$ \{expand, delete, combine, translocate, crossover\} \\
& with probs $\{0.80, 0.05, 0.05, 0.05, 0.05\}$. \\
\bottomrule
\end{tabularx}
\end{table*}

\subsection{Mathematical Reasoning}
\label{app:math_setup}

We use the problems of AIME 2024 and AIME 2025. A
trajectory is correct if the integer parsed from its final
\verb|\boxed{}| expression matches the reference answer. We report accuracy as the mean
$\pm$ standard error over 3 independent runs. Models are served with one B200 GPU with a context length of $32768$. Each problem is given a budget of 64 policy calls at 8,192 max tokens each, which includes the calls used for backward decomposition. Qwen3 models are run in non-thinking mode.

\paragraph{Forward search.}
All operators act on complete solutions through prompting.
Expansion generates a fresh solution. Deletion re-solves the problem given the checks that the parent failed and it is asked to remove something in the original solution. Combination re-solves the problem given a summary of
candidates from the pool.  We pick $\alpha=0.3$.

\paragraph{Backward search.}
The decomposition call asks the model to write executable
Python self-check predicates for the problem, e.g.\ parity,
divisibility, bound, or digit constraints on the answer. 

Hyperparameters are listed in Table~\ref{tab:hyperparams_math}.

\begin{table*}[h]
\centering
\caption{Hyperparameters for the mathematical reasoning experiment (AIME 2024--2025).}
\label{tab:hyperparams_math}
\begin{tabularx}{\textwidth}{@{}lX@{}}
\toprule
\textbf{Category} & \textbf{Hyperparameter $=$ Value} \\
\midrule
\textbf{Generation} &
temperature $=$ 0.7; max\_tokens $=$ 8192. \\
\midrule
\textbf{BES search} &
action\_types $=$ \{expand, delete, combine, translocate, crossover\} \\
& with probs $\{0.80, 0.05, 0.05, 0.05, 0.05\}$; \\
\bottomrule
\end{tabularx}
\end{table*}

\subsection{Competitive Programming}
\label{app:code_setup}

We use LiveCodeBench v6: all problems from contests held between
2025-01-04 and 2025-04-06. Accuracy is
computed with the official \texttt{lcb\_runner}. Models are served with one B200 GPU with a context length of $32768$. Each problem is given a budget of 64 policy calls at 4,096 max tokens each, which includes the calls used for backward decomposition. Qwen3 models are run in non-thinking mode.

\paragraph{Forward search.}
Candidates are complete programs, and the evolution operators use prompting. Crossover merges two parent programs, translocation transplants exactly one technique from one parent into the other, combination merges inspiration from both programs, and deletion identifies the part of a program that is unnecessary and removes it. Parents are drawn by Boltzmann sampling with $\tau$ annealed linearly from $0.30$ to $0.05$ over the budget.  We pick $\alpha=0.3$.

\paragraph{Backward search.}
The goal tree is generated by the same policy model. Each sub-goal is an executable check with an expected value.

Hyperparameters are listed in Table~\ref{tab:hyperparams_code}.

\begin{table*}[h]
\centering
\caption{Hyperparameters for the competitive programming experiment (LiveCodeBench v6).}
\label{tab:hyperparams_code}
\begin{tabularx}{\textwidth}{@{}lX@{}}
\toprule
\textbf{Category} & \textbf{Hyperparameter $=$ Value} \\
\midrule
\textbf{Generation} &
temperature $=$ 0.7; top\_p $=$ 0.8; top\_k $=$ 20; max\_tokens $=$ 4096. \\
\midrule
\textbf{BES search} &
action\_types $=$ \{expand, delete, combine, translocate, crossover\} \\
& with probs $\{0.80, 0.05, 0.05, 0.05, 0.05\}$; \\
\bottomrule
\end{tabularx}
\end{table*}

\section{Case Study}
\label{app:case}

We illustrate the \ours\ search process on a multi-hop question answering example from the MuSiQue dataset. The question is: \textit{``What is the record label of the artist who originally recorded Back to Bedlam?''} The correct answer is Custard Records. Figure~\ref{fig:case} shows the full search trace.

\begin{figure}[h]
\centering
\includegraphics[width=\linewidth]{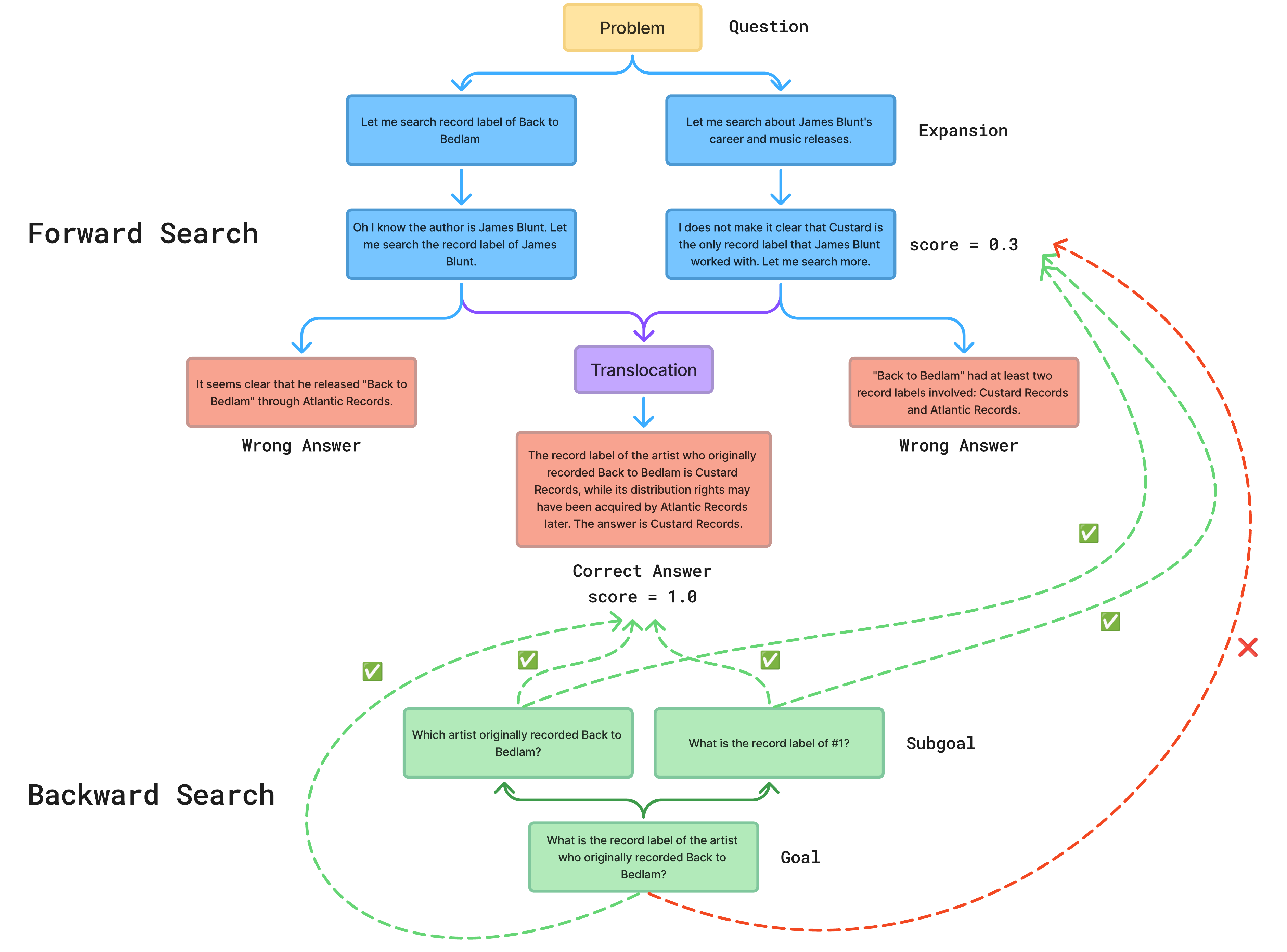}
\caption{Case study of \ours\ on a multi-hop reasoning problem. The forward search (top) explores two branches via expansion, both of which lead to wrong answers. A translocation operator then combines a reasoning step from the right branch into the left branch, producing a correct answer. The backward search (bottom) decomposes the original question into two sub-goals and provides dense verification feedback (green/red arrows) to guide parent selection.}
\label{fig:case}
\end{figure}

\paragraph{Backward search.} The backward search decomposes the original question into two sub-goals: (1) \textit{``Which artist originally recorded Back to Bedlam?''} and (2) \textit{``What is the record label of \#1?''}. Each sub-goal is equipped with a local verifier based on embedding similarity (Section~\ref{app:qa_setup}). This decomposition allows the search to track partial progress: a candidate that correctly identifies the artist (James Blunt) but retrieves the wrong label still receives a non-zero score, providing a useful signal for parent selection.

\paragraph{Forward search: expansion.} The search begins with two expansion branches from the root. The left branch searches for the record label of ``Back to Bedlam'' directly and concludes that the album was released through Atlantic Records, which is incorrect. The right branch takes a broader approach, searching for James Blunt's career and music releases, and discovers that both Custard Records and Atlantic Records were involved. However, it also fails to produce the correct final answer, arriving at a wrong conclusion.

\paragraph{Forward search: translocation.} At this point, the backward search scores both branches. The right branch receives a score of 0.3, indicating partial progress: it has identified relevant information about Custard Records but has not resolved the question correctly. The key moment occurs when the translocation operator replaces the reasoning step in the left branch (which concluded Atlantic Records) with the more nuanced reasoning step from the right branch (which noted both Custard Records and Atlantic Records). The resulting candidate inherits the left branch's identification of James Blunt as the artist and the right branch's awareness of the Custard Records connection. This recombined trajectory correctly concludes that the original record label is Custard Records, achieving a score of 1.0.

This example highlights two core advantages of \ours. First, the backward search provides dense intermediate feedback: even though both initial branches produce wrong final answers, the sub-goal scores distinguish the branch that has made more relevant progress (score 0.3), enabling informed parent selection. Second, the translocation operator constructs a correct trajectory by transplanting a single useful reasoning step from one branch into another. Neither branch alone would have reached the correct answer through further expansion, but their combination via translocation succeeds. This illustrates how evolution operators can discover solutions beyond what any single policy rollout can produce.

\section{Open Problem Solving with Open-Weight Models}
\label{app:open_weight}

The inference experiments in Section~\ref{sec:inference} use GPT-5 as the backbone model. To test whether the advantage of \ours\ depends on a closed-source backbone, we repeat the Circle Packing (Square) experiment with two open-weight models, GLM-5.2~\citep{glm5} and DeepSeek-V4-Pro~\citep{deepseekv4}, keeping all other configurations unchanged. For each backbone, all frameworks are run under the same backbone model, API budget, and evaluation protocol as in Table~\ref{tab:ops}, and we report the mean $\pm$ standard deviation and the best objective value over 3 runs.

As shown in Table~\ref{tab:open_weight}, \ours\ achieves the highest average and best objective values with every backbone. These results indicate that the gains of \ours\ come from the search procedure rather than from a particular backbone model.

\begin{table}[h]
\centering
\renewcommand{\arraystretch}{1.20}
\caption{Circle Packing (Square) with different backbone models. We report mean $\pm$ standard deviation and best objective values ($\uparrow$) over 3 runs. \textbf{Bold} denotes the best performing framework per backbone.}
\vspace{4pt}
\label{tab:open_weight}
\resizebox{\textwidth}{!}{%
\begin{tabular}{@{} l cc cc cc @{}}
\toprule
\multirow{2}{*}{\textbf{Strategy}}
& \multicolumn{2}{c}{\textbf{GPT-5}}
& \multicolumn{2}{c}{\textbf{GLM-5.2}}
& \multicolumn{2}{c}{\textbf{DeepSeek-V4-Pro}} \\
\cmidrule(lr){2-3}
\cmidrule(lr){4-5}
\cmidrule(lr){6-7}
& Avg. & Best & Avg. & Best & Avg. & Best \\
\midrule
\rowcolor{cellbase} OpenEvolve   & 2.418 $\pm$ .123 & 2.541 & 2.534 $\pm$ .012 & 2.541 & 2.526 $\pm$ .027 & 2.542 \\
\rowcolor{cellbase} GEPA         & 2.537 $\pm$ .004 & 2.541 & 2.541 $\pm$ .000 & 2.541 & 2.499 $\pm$ .059 & 2.541 \\
\rowcolor{cellbase} ShinkaEvolve & 2.478 $\pm$ .076 & 2.565 & 2.558 $\pm$ .026 & 2.595 & 2.458 $\pm$ .070 & 2.536 \\
\rowcolor{tabband}  ShinkaEvolve + \ours & \textbf{2.623} $\pm$ .014 & \textbf{2.632} & \textbf{2.617} $\pm$ .008 & \textbf{2.623} & \textbf{2.565} $\pm$ .042 & \textbf{2.600} \\
\bottomrule
\end{tabular}%
}
\end{table}

\section{Hyperparameter Sensitivity Analysis}
\label{app:sensitivity}

\noindent
\begin{minipage}[t]{0.58\textwidth}
\vspace{0pt}
We study the sensitivity of \ours\ to the blending weight $\alpha$ of the backward score (Eq.~\ref{eq:score}). All experiments are conducted on Circle Packing (Square) with GPT-5, keeping every other configuration as in Table~\ref{tab:hyperparams_inference}, and we report the best objective value over 3 runs.

In Eq.~\ref{eq:score}, $\alpha$ controls how much the score of an unsolved sub-goal relies on its own verifier versus the scores of its children. Performance is stable for $\alpha\in[0,0.7]$.
\end{minipage}
\hfill
\begin{minipage}[t]{0.38\textwidth}
\vspace{0pt}
\centering
\renewcommand{\arraystretch}{1.15}

\captionof{table}{Blending weight $\alpha$ of the backward score. }
\vspace{4pt}
\label{tab:sensitivity}

\begin{tabular}{@{} l c @{}}
\toprule
$\boldsymbol{\alpha}$ & \textbf{Best} \\
\midrule
\rowcolor{cellbase} $0.0$ & 2.630 \\
\rowcolor{tabband}  $0.3$ (default) & \textbf{2.632} \\
\rowcolor{cellbase} $0.7$ & 2.612 \\
\bottomrule
\end{tabular}
\end{minipage}

\section{Empirical Verification of Shell Confinement and Escape}
\label{app:shell_empirical}

Motivation~\ref{thm:shell} predicts that trajectories produced by expansion-only search concentrate in a narrow entropy shell, while evolution operators can produce trajectories outside it. We provide empirical evidence consistent with both predictions using operators on real model outputs. We sample $4{,}000$ trajectories from Llama-3.2-3B-Instruct (temperature $0.7$, top-$p$ $0.95$) on MATH-500~\citep{math500} and compute the per-token surprisal of the continuation after a shared prefix under the policy, $-\frac{1}{T}\log P(Y\mid x)$, where $T$ is the number of continuation tokens. Following the definition of the typical set in Motivation~\ref{thm:shell}, we define the empirical shell as the interval between the 10th and 90th percentiles of the expansion samples, $[0.079, 0.301]$ nats/token; its width of only $0.22$ nats/token around a median of $0.16$ shows how tightly expansion candidates are concentrated.

We then apply the forward operators to these trajectories and score the resulting candidates under the same policy: expansion continues a trajectory prefix with new policy samples ($1{,}000$ candidates), and the evolution operators inject a segment from a donor trajectory either by direct splicing at a step boundary or by prompting the policy to rewrite the continuation. Only $10.2\%$ of the expansion operations produce a candidate above the upper edge of the shell. In contrast, on the cells with evolution operators, $33.0\%$ of the direct-splice and $27.3\%$ of the prompted-rewrite operations do so, roughly three times the rate of expansion, and the escaping candidates lie on the high-surprisal side of the shell. The fraction of candidates below the lower edge drops from $8.0\%$ under expansion to $0.6\%$ and $1.7\%$. 
These results show that evolution increases the frequency
of candidates above the empirical shell boundary, while selecting the useful ones among them is the role of the backward search.

\section{Quality of the Embedding-Based Verifier on MuSiQue}
\label{app:verifier_quality}

The verifiers used in the multi-hop reasoning experiment (Appendix~\ref{app:qa_setup}) declare a sub-question covered when one of the trajectory's search queries has cosine similarity at least $\sigma_{\mathrm{cov}}=0.6$ with the sub-question in the embedding space. This is slightly different from whether the retrieved evidence actually \emph{answers} it. To quantify how well the two agree, we collect the search logs of \ours\ runs with Llama-3.1-8B-Instruct from the MuSiQue training data and sample $192$ (question, sub-question, trajectory node) triples from those problems. For each sampled triple, we compute the verifier output and ask GPT-5.6-luna to annotate whether the sub-question is answered by the passages retrieved along the trajectory. 

On these triples the verifier attains a precision of $0.78$, a recall of $0.74$, a specificity of $0.92$ and a negative predictive value of $0.90$. The verifier is therefore a reasonable proxy: it rarely withholds credit incorrectly and captures most of the true progress, which is sufficient for its role in \ours. A per-call LLM verifier would be more accurate but is too expensive here.

\section{Potential Limitations}

\label{app:lim}

We acknowledge the following limitations of the paper:

(1) \ours\ requires an objective reward signal to guide the search. It has not been tested on subjective evaluation tasks where such signals are difficult to obtain, such as academic writing.

(2) The backward search relies on the policy's ability to decompose problems into meaningful sub-goals. For very weak models, this decomposition capability is limited. 

(3) Due to resource constraints, our post-training experiments use relatively small models up to 8B parameters.

\end{document}